\documentclass{article}

\PassOptionsToPackage{numbers, compress}{natbib}

\usepackage[main, final]{neurips_2026}

\usepackage[utf8]{inputenc} 
\usepackage[T1]{fontenc}    
\usepackage{hyperref}       
\usepackage{url}            
\usepackage{booktabs}       
\usepackage{amsfonts}       
\usepackage{nicefrac}       
\usepackage{microtype}      
\usepackage{xcolor}         
\usepackage{amsmath}
\usepackage{amssymb}
\usepackage{mathtools}
\usepackage{amsthm}
\usepackage{enumitem}
\usepackage{subcaption}
\usepackage{thmtools}
\usepackage{makecell}
\usepackage{array}
\usepackage{pgfplots}
\usepackage{multirow}
\usepackage{booktabs}
\usepackage{dsfont}
\usepackage{wrapfig}

\pgfplotsset{compat=1.18}
\usetikzlibrary{arrows.meta,calc,intersections}

\usepackage[capitalize,noabbrev]{cleveref}
\theoremstyle{plain}
\newtheorem{theorem}{Theorem}[section]
\newtheorem{proposition}[theorem]{Proposition}

\theoremstyle{definition}
\newtheorem{definition}[theorem]{Definition}
\newtheorem{assumption}[theorem]{Assumption}
\theoremstyle{remark}
\newtheorem{remark}[theorem]{Remark}

\usepackage[textsize=tiny]{todonotes}

\usepackage[table]{xcolor}
\definecolor{lightgray}{gray}{0.9}
 
\title{Soft-Radial Projection for Constrained End-to-End Learning}

\author{%
  Philipp J.~Schneider \quad Daniel Kuhn \\[0.5em]
  EPFL \\[0.5em]
  \texttt{\{philipp.schneider, daniel.kuhn\}@epfl.ch}%
}

\begin{document}
\maketitle
\begin{abstract}
Constrained end-to-end learning trains neural networks whose outputs must satisfy feasibility constraints, such as resource, budget, or operational limits. While projection and optimization layers guarantee constraint satisfaction, boundary-based projections can introduce unfavorable optimization geometry: regions of the unconstrained output space are mapped to the same active face of the feasible set. This makes the resulting map locally rank-deficient, suppressing gradients in constrained directions and potentially impairing optimization---a bottleneck known as gradient saturation. We propose \emph{Soft-Radial Projection (SRP)}, a differentiable reparameterization layer that maps unconstrained network outputs into the relative interior of a convex feasible set by smoothly rescaling rays from a strictly feasible anchor point. Unlike hard radial or orthogonal projections, the resulting map is one-to-one and has a full-rank Jacobian almost everywhere, while guaranteeing strict feasibility. We prove that constrained networks equipped with SRP retain universal approximation guarantees. For polytopes and balls, the layer admits a closed-form forward pass, while for general convex sets, the ray--boundary intersection reduces to a one-dimensional search. Experiments on decision-focused learning benchmarks demonstrate strong solution quality at low computational cost, with particularly favorable results in portfolio optimization and substantially lower inference latency than differentiable optimization layers in the evaluated settings.
\end{abstract}

\section{Introduction}
\label{sec:intro}

Many decision-making systems operate under hard constraints during both training and deployment---for example, safety envelopes in autonomous driving, actuator limits in robotics \citep{brunke2022safe}, or budget and capacity constraints in operations. While neural networks are powerful function approximators, they are not inherently constraint-aware, and naive usage can lead to \emph{infeasible} outputs precisely in scenarios where constraint violations have serious consequences. We formalize constrained learning objectives and propose a strict-interior projection layer that enforces feasibility \emph{by construction}.

\textbf{Setup.} Let $\mathcal{Z}\subseteq\mathbb{R}^d$ be the input space, $\mathcal{Y}\subseteq\mathbb{R}^n$ the target space, and $\mathcal{C}\subseteq\mathbb{R}^n$ the feasible set of decisions, assumed closed and convex with a nonempty interior. This assumption is without loss of generality, as we can restrict the ambient space to the affine hull of the convex set and formulate all rank and Jacobian statements in the corresponding coordinates. Data are drawn from a distribution $\mathbb{P}$ on $\mathcal{Z}\times\mathcal{Y}$ with marginal $\mathbb{P}_Z$ on $\mathcal{Z}$. We define the policy space $\Pi$ as the set of all measurable mappings $\pi: \mathcal{Z} \to \mathbb{R}^n$ that map an input $z$ to a decision.

\textbf{Supervised learning.} Let $(z,y)\sim\mathbb{P}$. We learn an optimal policy by solving
\begin{equation}
    \label{eq:supervised}
    \begin{array}{cl}
        \displaystyle \min_{\pi\in\Pi} & \displaystyle \left\{ \mathbb{E}_{(z,y)\sim\mathbb{P}}\big[\ell\big(\pi(z),y\big)\big]
        : \pi(z)\in\mathcal{C} \text{ for $\mathbb{P}_Z$-a.e.\ } z \right\},
    \end{array}
\end{equation}

where $\ell:\mathcal{C}\times\mathcal{Y}\to\mathbb{R}$ is a chosen loss. This formulation covers standard regression tasks as well as imitation learning, where the label $y$ represents a hindsight-optimal decision $\pi^\star(z)$. In such cases, $\ell$ quantifies the distance to the optimal decision or a regret surrogate.

\textbf{Task-driven learning.} Without labels, we measure performance directly via a task cost $c:\mathcal{Z}\times\mathcal{C}\to\mathbb{R}$ and optimize
\begin{equation}
    \label{eq:task}
    \begin{array}{cl}
        \displaystyle \min_{\pi\in\Pi} & \displaystyle \left\{\mathbb{E}_{z\sim\mathbb{P}_Z}\big[c\big(z, \pi(z)\big)\big] : \pi(z)\in\mathcal{C} \text{ for $\mathbb{P}_Z$-a.e.\ } z\right\}.
    \end{array}
\end{equation}
This formulation aligns the training objective with the downstream operational cost, thereby circumventing the potential objective mismatch inherent in surrogate supervised targets~\citep{donti2017task, elmachtoub2022smart, rychener2023end, chenreddy2024end}.

\textbf{Feasibility by construction.} A common way to enforce feasibility is to compose an unconstrained candidate map $g:\mathcal Z\to\mathbb R^n$ with a fixed constraint-enforcement map $\operatorname{Proj}:\mathbb R^n\to\mathcal C$, yielding a feasible policy $\pi(z)=\operatorname{Proj}(g(z))$. This avoids penalties and post-hoc repair, but the geometry of $\operatorname{Proj}$ strongly affects optimization. Boundary-valued maps, such as orthogonal projections or hard radial/gauge projections, enforce feasibility by collapsing infeasible points onto $\partial\mathcal C$. As a result, they are typically many-to-one outside the feasible set and can induce rank-deficient Jacobians, suppressing gradient components in the collapsed directions.

The standard example is the orthogonal projection $P(u)=\arg\min_{v\in\mathcal C}\|u-v\|^2$. Related hard radial and gauge-based layers project along rays from an interior anchor to the boundary~\citep{tabas2022safe,mhammedi2022,tordesillas2023rayen}. These constructions are computationally attractive and form an important basis for our work. However, these maps remain many-to-one outside the feasible set: hard radial and gauge-based projections collapse points along exterior rays, while orthogonal projection collapses points along outward normal directions. Consequently, the gradient-saturation phenomenon illustrated in Figure~\ref{fig:projection} is not specific to the orthogonal projection.

\begin{figure*}[!htbp]
    \centering
    \begin{subfigure}[b]{0.48\textwidth}
        \centering
        \includegraphics[width=\textwidth]{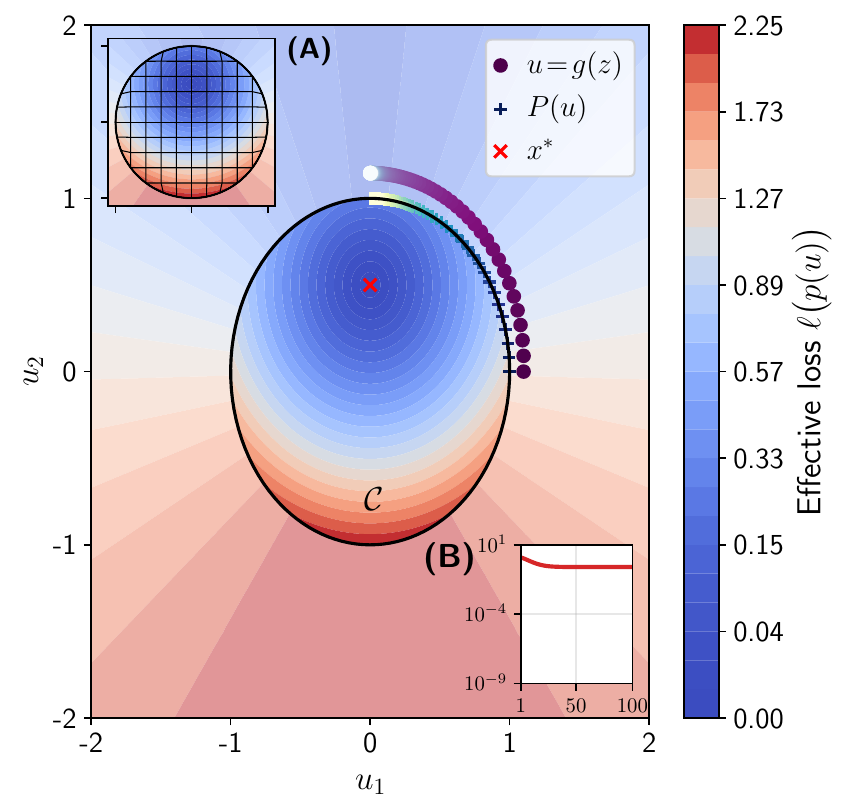}
        \caption{Orthogonal projection $P(u)$.}
        \label{fig:orthogonal-projection}
    \end{subfigure}
    \begin{subfigure}[b]{0.48\textwidth}
        \centering
        \includegraphics[width=\textwidth]{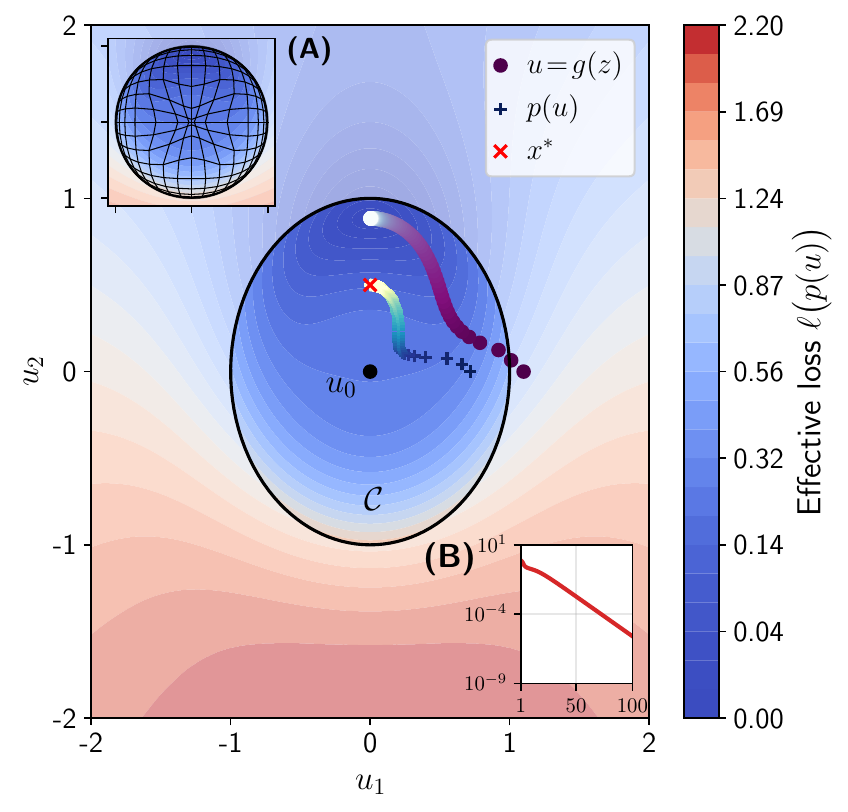}
        \caption{Soft-radial projection $p(u)$ into $\operatorname{Int}(\mathcal C)$.}
        \label{fig:soft-projection}
    \end{subfigure}
    \caption{Impact of projection geometry on optimization. Comparison of (a) Orthogonal and (b) Soft-Radial Projection on a 2D constrained task (target $\times$). Main plots visualize the trajectory of the \textit{unconstrained candidate} ($\circ$) versus the \textit{feasible decision} ($+$). Insets show the coordinate grid warping (A) and training loss (B). Soft-radial projection preserves nonzero radial derivatives for infeasible inputs.}
    \label{fig:projection}
\end{figure*}

\textbf{Soft-Radial Projection to the interior.}
We propose \emph{Soft-Radial Projection} (SRP), a strict-interior alternative to hard radial clipping. Starting from the same ray-based geometry as gauge and radial projection layers, SRP replaces boundary clipping with a soft radial contraction that maps every $u\in\mathbb R^n$ into $\operatorname{Int}(\mathcal C)$. Under mild assumptions, the resulting map is injective, defines a homeomorphism $p:\mathbb R^n\to\operatorname{Int}(\mathcal C)$, and has an invertible Jacobian wherever it is differentiable, avoiding the many-to-one collapse of hard radial, gauge-based, and orthogonal projections. For polytopes and balls, SRP admits a closed-form forward pass, while for general convex sets, the ray--boundary intersection reduces to a one-dimensional search.

\textbf{Related work.}
The challenge of obtaining feasibility guarantees for neural networks has attracted broad interest, from safety-critical control~\citep{garcia2015comprehensive} to decision-focused analytics~\citep{kotary2021}. Existing methodologies split into \textit{optimization-based layers}~\citep[\textit{e.g.},][]{amos2017optnet,agrawal2019differentiable,dontidc3}, which differentiate through iterative solvers, and \textit{constructive layers} that enforce feasibility by design~\citep[\textit{e.g.},][]{konstantinov2023,liang2024,min2024hardnet,tordesillas2023rayen}. Most closely related are ray- and gauge-based constructive layers, such as RAYEN~\citep{tordesillas2023rayen} and the ray-projection method of \citet{konstantinov2023}, which also exploit radial geometry. Our contribution is to replace hard boundary clipping with a strict-interior homeomorphic reparameterization, preserving the computational appeal of ray-based enforcement while avoiding the rank-deficient projection Jacobians induced by clipping. A detailed comparison is provided in Appendix~\ref{sec:related-work}.

\textbf{Contributions.} Our main contributions are:
\begin{itemize}[leftmargin=*, itemsep=0pt, topsep=0pt, partopsep=0pt, parsep=0pt]
    \item We introduce \emph{Soft-Radial Projection} (SRP), a ray-based enforcement layer for closed convex sets that replaces boundary clipping with a homeomorphic reparameterization of $\mathbb{R}^n$ onto $\operatorname{Int}(\mathcal{C})$.
    \item We prove that SRP has an invertible Jacobian almost everywhere and establish value equivalence, interior stationarity, bounded-latent value gaps, and constrained universal approximation.
    \item We demonstrate strong solution quality at low computational cost on decision-focused learning benchmarks, with particularly favorable results in portfolio optimization.
\end{itemize}

We proceed by formally introducing the soft-radial projection mechanism.

\textbf{Notation.}
We use $\|u\|$ for the Euclidean norm of vectors and $\|A\|$ for the spectral norm of matrices. For functions $\psi: \mathcal{Z} \to \mathbb{R}^n$, we denote the uniform norm by $\|\psi\|_\infty \coloneqq \sup_{z \in \mathcal{Z}} \|\psi(z)\|$. We write $\operatorname{Int}(\mathcal{C})$ and $\partial\mathcal{C}$ for the interior and boundary of the feasible set. The Jacobian of a differentiable map $\phi$ is denoted by $J_\phi$. We reserve $u$ for unconstrained vectors (pre-projection) and use $x$ or $\pi(z)$ for feasible vectors in $\mathcal{C}$.

\section{Soft-Radial Projection Layer}
\label{sec:soft-projection}

Many models output a raw vector in $\mathbb R^n$ that must satisfy hard constraints at inference time---\textit{e.g.}, neural networks and decoders, control policies, differentiable optimization layers, or amortized solvers. We place a \emph{soft-radial projection layer} $p:\mathbb R^n\to\operatorname{Int}(\mathcal C)$ downstream of such modules so that, regardless of how the raw output is produced (even if infeasible), the final decision lies in the feasible set. 
\begin{wrapfigure}{t}{0.5\linewidth}
    \vspace{-1.2em}
    \centering
    \includegraphics[width=1.0\linewidth]{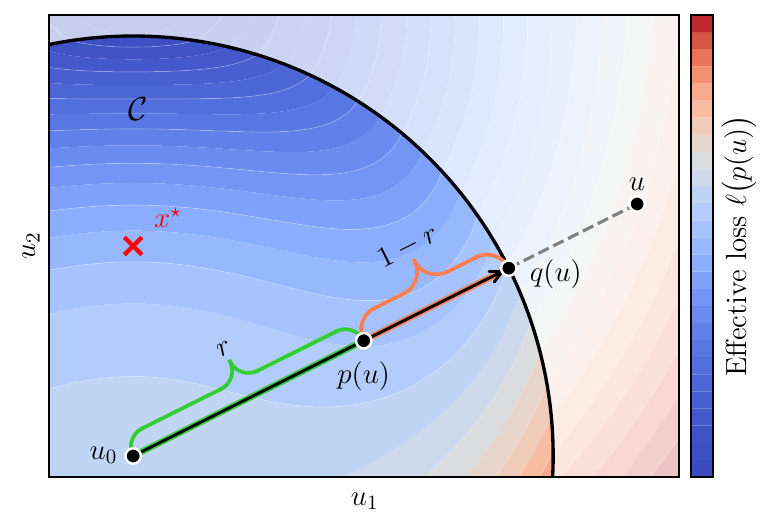}
    \vspace{-1.0em}
    \caption{Geometric intuition. For an infeasible input $u$, the ray from anchor $u_0$ through $u$ intersects the boundary $\partial\mathcal C$ at the hard radial projection $q(u)$. Our soft-radial mapping \(p(u)\) strictly enforces feasibility by smoothly interpolating along the segment \([u_0, q(u)]\) via the radial weight \(r\).}
    \label{fig:geometric-intuition}
    \vspace{-1.0em}
\end{wrapfigure}
Concretely, we instantiate the unconstrained candidate $g$ as a parameterized map $g_\theta:\mathcal Z\to\mathbb R^n$ (\textit{e.g.}, a neural network). Given the raw output $u=g_\theta(z)$, we set the policy as $\pi_\theta = p\circ g_\theta$, which ensures $\pi_\theta(z)\in\operatorname{Int}(\mathcal C)$ during training and deployment without penalty tuning or post-hoc repair. The learning objectives~\eqref{eq:supervised} and \eqref{eq:task} are then instantiated by composing with $p$ (\textit{e.g.}, $\ell(\pi_\theta(z),y)$ or $c(z,\pi_\theta(z))$). This composition turns any base predictor into a constraint-satisfying model while keeping gradients usable for end-to-end training. The remainder of this section develops the geometry and regularity of $p$ needed for gradient-based methods. We first formalize the soft-radial projection $p$ and its raywise components. The construction relies on a smooth contraction toward an anchor point, mapping the entire ambient space into the interior of $\mathcal C$ while retaining differentiability almost everywhere.

\begin{definition}[Soft-Radial Projection]
\label{def:soft-projection}
Let $\mathcal C \subseteq \mathbb R^n$ be a closed, convex set with nonempty interior, and fix an anchor point $u_0 \in \operatorname{Int}(\mathcal C)$. Let $r:\mathbb R_{+}\to[0,1)$ be continuous and strictly increasing with $r(0)>0$ and $\lim_{\rho\to\infty} r(\rho)=1$.

For any $u\in\mathbb R^n$, the \emph{soft-radial projection} of $u$ onto $\mathcal C$ is 
\[
p(u) \;=\; u_0 \;+\; r\!\big(\|u-u_0\|^2\big)\,\big(q(u)-u_0\big),
\]
where the \emph{hard radial projection} $q(u)$ is defined by
$q(u)=u$ if $u\in\mathcal C$, and $q(u)=u_0+\alpha^\star(u)(u-u_0)$ otherwise, with the scaling factor
$\alpha^\star(u)=\sup\{\alpha\in[0,1]: u_0+\alpha(u-u_0)\in\mathcal C\}\in[0,1]$.
\end{definition}

\begin{remark}
The mapping $q$ projects $u$ onto $\mathcal C$ along the ray emanating from the anchor $u_0$, in contrast to the usual Euclidean (orthogonal) projection. The soft-radial projection $p$ then interpolates between $u_0$ and $q(u)$, with interpolation factor governed by $r(\|u-u_0\|^2)$ (see Fig.~\ref{fig:geometric-intuition}). If we drop the strict-interior requirement and allow $r\equiv1$, then $p(u)=q(u)$ and the soft-radial projection takes the form of the hard radial projection.
\end{remark}

\textbf{Coordinate convention.}
The construction of \(p\) is invariant under translations of both the input and the constraint set: replacing \(u\) by \(u-u_0\) and \(\mathcal C\) by \(\mathcal C-\{u_0\}\) leaves the construction unchanged up to translation. Hence, throughout the subsequent geometric and optimization analysis and the corresponding proofs, we work in translated coordinates and assume without loss of generality that \(0\in\operatorname{Int}(\mathcal C)\) is the anchor. Under this convention, \(q(u)=u\) for \(u\in\mathcal C\), \(q(u)=\alpha^\star(u)u\) otherwise, and \(p(u)=r(\|u\|^2)q(u)\). The general-anchor form is recovered by translating back by \(u_0\).

\textbf{Regularity of components.}
We first record basic properties of the ray map $q$ and the scalar $\alpha^\star$. These facts will be used later to build a raywise parametrization of $\operatorname{Int}(\mathcal C)$ and to analyze $\ell\circ p$ with first-order methods. In particular, continuity and local Lipschitz continuity of $q$ provide the regularity needed for the subsequent analysis of $p$.

\begin{restatable}{lemma}{RayRegularity}\label{lem:ray-regularity}
\textnormal{(Ray intersection and continuity).} Let $\mathcal C\subset\mathbb R^n$ be nonempty, closed, and convex with nonempty interior, and assume $0\in\operatorname{Int}(\mathcal C)$. For any $u\in\mathbb R^n$, the set $\{\alpha\in[0,1] : \alpha u\in\mathcal C\}$ is a nonempty closed interval $[0,\alpha^\star(u)]$ for a unique $\alpha^\star(u)\in[0,1]$, and $q(u)\in\mathcal C$. Moreover, $\alpha^\star$ is globally Lipschitz continuous and $q$ is locally Lipschitz (and therefore differentiable almost everywhere by Rademacher's theorem; see, e.g.,~\cite{federer1969geometric}, Thm.~3.1.6).
\end{restatable}

\textbf{Structural properties of $p$.} We now analyze the global geometry of $p$. A critical requirement is that $p$ must preserve the \emph{representational capacity} of the model $g$. Unlike orthogonal projections, which suffer from dimensional collapse by mapping the exterior space onto the boundary surface (a many-to-one mapping), our construction guarantees that $p$ induces a \emph{one-to-one parametrization} of $\operatorname{Int}(\mathcal C)$ by strictly rescaling radii along rays emanating from the origin.

\begin{assumption}[Radial monotonicity]\label{ass:radial}
The function $r:\mathbb R_+\to[0,1)$ is $C^1$, strictly increasing with $r'(\rho)>0$ for all $\rho>0$, and satisfies $r(0)>0$ and $\lim_{\rho\to\infty} r(\rho)=1$.
\end{assumption}

\noindent Assumption \ref{ass:radial} ensures that the soft-radial projection is strictly monotonic along rays, providing the topological foundation for the one-to-one mapping. This prevents distinct points on the same ray from collapsing to the same interior location. Specifically, the scalar map $t \mapsto r(t^2)t$ has derivative $
\frac{d}{dt}\bigl[ r(t^2)t \bigr] \;=\; r(t^2) + 2t^2 r'(t^2)$.
Since $r$ and $r'$ are positive, this derivative is strictly positive for all $t \ge 0$. Beyond the boundary in a direction $v$, the radial profile is $r(t^2)\bar t(v)$, whose derivative $2t\,r'(t^2)\bar t(v)$ is also strictly positive. Thus, $p$ is strictly monotone along each ray and, since it preserves ray directions, is injective.

The mapping $p$ admits a radial decomposition into angular and radial components, simplifying the analysis. Since the projection preserves the direction of the ray $\mathbb{R}_+ v$ (the angular component), its behavior is characterized by its action on the scalar distance from the origin—the radial component.

\begin{definition}[Boundary distance]
Fix a unit direction $v\in\mathbb R^n$ ($\|v\|=1$). We define the \emph{boundary distance} $\bar t(v) \in (0, \infty]$ as the distance from the origin to the boundary of $\mathcal C$ along $v$, given by $\bar t(v) \coloneqq \sup\{t \ge 0 : t v \in \mathcal C\}$.
\end{definition}

With the boundary established, the radial profile of the projection is defined by the function $\psi_v$. This map distinguishes between the interior regime, where the ray remains within $\mathcal C$, and the exterior regime, where the input is scaled relative to the boundary.

\begin{definition}[Scalar projection map]\label{def:scalar-map}
For a fixed unit direction $v$, define the map $\psi_v:[0,\infty)\to[0,\infty)$ by $\psi_v(t) \coloneqq r(t^2)\min\{t,\bar t(v)\}.$
\end{definition}

By construction, the soft-radial projection satisfies $p(tv) = \psi_v(t)v$. With this decomposition, we now establish the regularity of these components. The following lemma verifies that the boundary distance varies continuously with direction and that the scalar projection is strictly monotonic---properties that are essential for proving the soft-radial projection is a global bijection.

\begin{restatable}{lemma}{RegComponents}\label{lem:reg-comp}
\textnormal{(Regularity of components).} The boundary distance map $\bar t$ is continuous on the unit sphere $\{v:\|v\|=1\}$ (with values in $(0, \infty]$). Moreover, under Assumption~\ref{ass:radial}, for any unit direction $v$, the scalar map $\psi_v$ is continuous and strictly increasing, and its range is exactly $[0, \bar t(v))$.
\end{restatable}

The raywise analysis in Lemma~\ref{lem:reg-comp} provides the sufficient conditions to characterize the global topology of $p$. We now formally state the main result: $p$ is a homeomorphism, acting as a reversible deformation of the entire ambient space $\mathbb R^n$ onto the feasible interior $\operatorname{Int}(\mathcal C)$.

\begin{restatable}{theorem}{Homeomorphism}\label{thm:homeomorphism}
\textnormal{(Homeomorphism).} Under Assumption~\ref{ass:radial}, the soft-radial projection $p:\mathbb R^n\to\operatorname{Int}(\mathcal C)$ is a homeomorphism.
\end{restatable}

This theorem establishes a continuous, bijective correspondence between the latent output space $\mathbb R^n$ and the feasible interior $\operatorname{Int}(\mathcal C)$, with a continuous inverse.

\textbf{Differential regularity.} Finally, to enable gradient-based learning, we establish differentiability of the soft-radial projection and full-rank Jacobians wherever differentiable.

\begin{restatable}{theorem}{DiffInvert}\label{thm:differentiability}
\textnormal{(Differentiability and invertibility).} Under Assumption~\ref{ass:radial}, the soft-radial projection $p$ is differentiable almost everywhere on $\mathbb R^n$. The Jacobian $J_p(u)$ is invertible at every point $u$ where it exists.
\end{restatable}

\noindent The invertibility of $J_p$ ensures that the mapping has \emph{full rank} almost everywhere. Consequently, soft-radial projection avoids the rank deficiency characteristic of orthogonal and hard radial projections.
\vspace{-0.5\baselineskip} 

\section{Optimization Guarantees}
\label{sec:opt-guarantees}

After establishing the geometry of the projection layer, we now analyze its implications for \emph{unconstrained optimization}. Consider a standard constrained task where we seek to minimize a loss $\ell(x)$ over the feasible set $\mathcal{C}$. By equipping the model with the soft-radial projection, we transform this problem into the unconstrained minimization of the composite objective $f(u) \coloneqq \ell(p(u))$, where $u\in\mathbb{R}^n$ denotes the unconstrained candidate vector, for instance the model output $u=g(z)$.

Rather than developing specialized optimization algorithms, our goal is to translate the geometric properties of $p$ into first-order optimization properties of the composite objective. Similar unconstrained-training viewpoints appear in prior optimization-layer and enforcement-layer approaches \citep{amos2017optnet,dontidc3,tordesillas2023rayen,min2024hardnet}. Our specific contribution is to characterize what the strict-interior soft-radial geometry implies for optimal value equivalence, interior stationarity, the value gap under bounded latent norms, and the inherent failure of global PL-type bounds.

\textbf{Optimal value equivalence.}
Our first requirement is that the unconstrained formulation should preserve the value of the original constrained problem. Since $p$ is a homeomorphism from $\mathbb R^n$ onto $\operatorname{Int}(\mathcal C)$, the unconstrained parametrization covers exactly the feasible interior. Continuity of $\ell$ then ensures that optimizing over the interior has the same infimum as optimizing over the closed feasible set.

\begin{restatable}{theorem}{EquivOptVal}\label{thm:value}
\textnormal{(Equivalence of optimal values).} Let $\ell:\mathcal C\to\mathbb R$ be continuous. With $f \coloneqq \ell\circ p$, we have
\[
\inf_{u\in\mathbb R^n} f(u)
\;=\;
\inf_{x\in\operatorname{Int}(\mathcal C)} \ell(x)
\;=\;
\inf_{x\in\mathcal C}\ell(x).
\]
Moreover, $\arg\min f\neq\emptyset$ if and only if $\arg\min_{\mathcal C}\ell$ intersects $\operatorname{Int}(\mathcal C)$. In that case, $u^\star\in\arg\min f$ if and only if $p(u^\star)\in\arg\min_{\mathcal C}\ell$.
\end{restatable}

\textbf{First-order correspondence.}
Beyond agreement of optimal values, optimization algorithms require that first-order stationarity of the unconstrained objective $f=\ell\circ p$ be interpretable in terms of first-order conditions for the original constrained problem. Since $p(\mathbb R^n)\subset \operatorname{Int}(\mathcal C)$, the correspondence we establish is necessarily an \emph{interior} one: it relates stationary points of $f$ to points $x\in\operatorname{Int}(\mathcal C)$ satisfying $\nabla \ell(x)=0$.

\begin{restatable}{proposition}{CorrStatPoints}\label{prop:stationary}
\textnormal{(Correspondence of stationary points).} Assume $\ell$ is $C^1$ on $\operatorname{Int}(\mathcal C)$. Wherever $p$ is differentiable, with $f=\ell\circ p$ we have $\nabla f(u)=J_p(u)^\top\,\nabla\ell\big(p(u)\big)$.
\end{restatable}

Since $J_p(u)$ is invertible almost everywhere by Theorem~\ref{thm:differentiability}, the stationary conditions are equivalent at such points $\nabla f(u) = 0\Longleftrightarrow\nabla\ell\big(p(u)\big) = 0$. Thus, the full-rank Jacobian ensures that the unconstrained landscape introduces no projection-induced spurious stationary points within the feasible interior wherever $p$ is differentiable. If the constrained problem has no interior minimizer, any minimizing sequence for $f$ must satisfy $\|u_k\|\to\infty$, since the soft-radial projection approaches boundary optima only asymptotically.

\textbf{Value gap under bounded latent norm.}
As reaching a boundary minimizer strictly requires the unconstrained prediction to diverge, we must consider practical scenarios where network outputs are constrained by a latent norm budget $M$, for instance through weight decay, clipping, or architectural limits. This restricts the reachable objective values to $\ell_M^\star \coloneqq \inf_{\lVert u \rVert \le M} f(u)$. Proposition~\ref{prop:budget_gap} quantifies the resulting suboptimality gap relative to the true constrained minimum $\ell^\star \coloneqq \inf_{x\in\mathcal C}\ell(x)$.

\begin{restatable}{proposition}{BudgetGap}\label{prop:budget_gap}
\textnormal{(Value gap under a latent norm budget).} Assume the problem admits a boundary minimizer $x^\star\in \arg\min_{x\in\mathcal C}\ell(x)\cap\partial\mathcal C$ with $x^\star=\bar t(v^\star)\,v^\star$ for some unit vector $v^\star$, and assume that $\ell$ is globally $L$-Lipschitz on $\mathcal C$. Define the boundary deficit $\delta_M(v)\coloneqq\bar t(v)-\psi_v(M)$. Then $0\le \ell_M^\star-\ell^\star \le L\,\delta_M(v^\star)$.
In particular, if $M\ge \bar t(v^\star)$, then $\delta_M(v^\star)=(1-r(M^2))\,\bar t(v^\star)$, and hence $0\le \ell_M^\star-\ell^\star \le L(1-r(M^2))\,\bar t(v^\star)$.
\end{restatable}

Proposition~\ref{prop:budget_gap} makes the trade-off explicit: any mechanism ensuring $\sup_z\|g_\theta(z)\|\le M$ restricts the attainable decisions and induces at most $L\,\delta_M(v^\star)$ suboptimality. Once $M\ge \bar t(v^\star)$, the bias is controlled directly by the radial saturation error $1-r(M^2)$. Conversely, if the constrained problem admits an interior minimizer $x^\circ\in\operatorname{Int}(\mathcal C)$, then $u^\circ \coloneqq p^{-1}(x^\circ)$ has finite norm, so any budget $M\ge \|u^\circ\|$ incurs no bias.

\textbf{Global optimality for convex losses.}
When $\ell$ is convex, any \emph{interior} stationary point is globally optimal. Since $p(\mathbb R^n)=\operatorname{Int}(\mathcal C)$ and Proposition~\ref{prop:stationary} establishes the equivalence of stationary conditions wherever $J_p$ is invertible, stationary points of $f$ at such locations correspond to global minimizers of $\ell$ over $\mathcal C$.

\begin{restatable}{corollary}{GlobOptInt}\label{cor:global-opt}
\textnormal{(Global optimality of interior stationary points for convex losses).} Assume $\ell$ is convex and $C^1$ on $\operatorname{Int}(\mathcal C)$, and let $f \coloneqq \ell\circ p$. If $p$ is differentiable at $u$ and $\nabla f(u)=0$, then $p(u)$ is a global minimizer of $\ell$ over $\mathcal C$.
\end{restatable}

\noindent\emph{Remark (Anchor).}
At the anchor $u_0\in\operatorname{Int}(\mathcal C)$ we have $J_p(u_0)=r(0)I$, hence $\nabla f(u_0)=r(0)\,\nabla\ell(u_0)$. Thus, the condition $r(0)>0$ in Assumption~\ref{ass:radial} prevents the anchor from becoming an artificial stationary point.

\textbf{On PL-type conditions and bounded iterates.}
Because the soft-radial map saturates as $\|u\|\to\infty$, gradients can become arbitrarily small even when the objective value remains bounded away from its infimum. This rules out a global Polyak--\L{}ojasiewicz (PL) inequality in general.

\begin{restatable}{lemma}{AbsencePL}
\label{lem:pl}\textnormal{(Absence of a global PL inequality in general).}
Even if $\mathcal C\coloneqq B_1(0)$, $\ell(x) \coloneqq \|x\|^2$ and $u_0 \coloneqq 0$, there exists no function $r$ satisfying Assumption~\ref{ass:radial} such that $f=\ell\circ p$ satisfies a global PL inequality on $\mathbb R^n$.
\end{restatable}

Although the global PL condition fails, this obstruction is global. In practice, optimization is confined to a bounded region, for example through weight decay, clipping, trust regions, or projected updates. On bounded subsets of its input space, $p$ is locally Lipschitz and differentiable a.e.~with a full-rank Jacobian wherever differentiable. Consequently, whenever the composite training objective $F(\theta)\coloneqq\mathbb E_z\big[\ell(p(g_\theta(z)))\big]$ satisfies the usual compactness, smoothness, and integrability assumptions, it falls within standard stochastic first-order frameworks, yielding nonconvex SGD convergence in the smooth regime. In the nonsmooth regime, standard stochastic subgradient results guarantee convergence to Clarke-stationary points under local-Lipschitzness and standard stochastic-oracle assumptions~\cite[Cor.~5.9]{davis2020stochastic}. Importantly, smoothness of $F$ is an assumption on the expected composite objective itself; it is not guaranteed merely because the input distribution is absolutely continuous, since the network outputs may still concentrate on nonsmooth regions of $p$. We defer the formal regularity assumptions, proof intuition, and precise statement of Proposition~\ref{prop:convergence} to Appendix~\ref{app:proofs-opt-guarantees}.

\textbf{Robustness to degeneracy.} A potential concern is the degenerate scenario where the network output distribution concentrates on a nondifferentiable set of the soft-radial map. In this case, the analysis of convergence requires additional assumptions from nonsmooth optimization, as discussed above. Unlike orthogonal projections that collapse exterior regions onto lower-dimensional boundary faces, soft-radial projection maps strictly to the interior and maintains an invertible Jacobian wherever differentiable. This avoids the rank deficiency characteristic of orthogonal projections, while accounting for the possibility of nondifferentiable network outputs.

\section{Constrained Universal Approximation}
\label{sec:const-ua}
We next demonstrate that soft-radial projection preserves the expressivity of the base model. While constrained universality has been established for input-dependent affine layers~\citep{min2024hardnet} and probabilistic transformer constructions~\citep{kratsios2022universal}, our guarantee relies on a distinct topological mechanism: the global homeomorphism $p:\mathbb{R}^n \to \operatorname{Int}(\mathcal{C})$.
Let $\mathcal{Z} \subset \mathbb{R}^d$ be compact and let $\mathcal{G} \coloneqq \{g_\theta: \mathcal{Z} \to \mathbb{R}^n\}$ be a universal approximator in $C(\mathcal{Z}, \mathbb{R}^n)$, such as sufficiently wide or deep neural networks~\citep{cybenko1989,hornik1991approximation}. We show that the constrained class $p \circ \mathcal{G}$ can uniformly approximate any continuous target $h: \mathcal{Z} \to \mathcal{C}$. Since $p$ maps strictly into the open set $\operatorname{Int}(\mathcal{C})$, the key point is to handle targets whose image intersects the boundary $\partial\mathcal{C}$.
\begin{restatable}{theorem}{ConstUnvAppr}\label{thm:ua}
\textnormal{(Universal approximation on $\mathcal{C}$).} Let $\mathcal{Z} \subset \mathbb{R}^d$ be compact and let $\mathcal{G}$ be a universal approximator on $\mathcal{Z}$. For every continuous function $h: \mathcal{Z} \to \mathcal{C}$ and every $\varepsilon > 0$, there exists a function $g \in \mathcal{G}$ such that the constrained predictor $p \circ g$ satisfies $\sup_{z \in \mathcal{Z}} \|\, p(g(z)) - h(z) \,\| \;\le\;\varepsilon$.
\end{restatable}
\vspace{-0.5\baselineskip}
\noindent\textit{Proof intuition.} The argument proceeds by density and composition. Convexity of $\mathcal{C}$ ensures that strictly interior functions are dense in $C(\mathcal{Z}, \mathcal{C})$, so any target $h$ admits a uniformly close interior approximation $h_\varepsilon$. Since $p$ is a homeomorphism onto $\operatorname{Int}(\mathcal{C})$, there exists a continuous pre-image $\phi$ such that $p \circ \phi = h_\varepsilon$. The result follows by applying the universality of $\mathcal{G}$ to approximate $\phi$.

\noindent\emph{Remark (Unbounded sets $\mathcal{C}$).} Compactness of $\mathcal{Z}$ ensures that $h(\mathcal{Z})$ is bounded even if $\mathcal{C}$ is unbounded. Consequently, the relevant pre-images in $\mathbb{R}^n$ remain within a compact subset, where uniform approximation bounds apply.

\textbf{Approximation stability.}  Uniform error bounds transfer to the constrained model up to a constant factor. Let $K \subset \mathbb{R}^n$ be compact and let $L_K$ be a Lipschitz constant of $p$ on $K$. For any continuous $\phi: \mathcal{Z} \to \mathbb{R}^n$ and any $g \in \mathcal{G}$ satisfying $\phi(\mathcal{Z}), g(\mathcal{Z}) \subset K$, we have $
\sup_{z \in \mathcal{Z}} \| p(g(z)) - p(\phi(z)) \|
\le
L_K \sup_{z \in \mathcal{Z}}\| g(z) - \phi(z) \|$.
Thus, composition with $p$ preserves the expressivity of $\mathcal G$ and transfers uniform error bounds from $\mathcal G$ to the constrained class $p\circ\mathcal G$, up to the Lipschitz factor $L_K$.

\section{Numerical Experiments}
\label{sec:numerical-experiments}

While Section~\ref{sec:intro} emphasized safety, strict constraints are equally critical in operations where violations are prohibitive. We evaluate our framework on formulation \eqref{eq:task} via portfolio optimization and ride-sharing dispatch; see Appendix~\ref{app:extended-numerical-results} for extended experimental details. For our implementation, we treat the projection as a differentiable layer, allowing us to train the model end-to-end by backpropagating gradients directly through the projection operation to the base network $g_\theta$.\footnote{Code: \url{https://github.com/RAO-EPFL/SRP}}
\vspace{-0.5\baselineskip} 
\subsection{Implementation Details}
\label{subsec:implementation-details}

A critical computational advantage of the soft-radial projection is that the hard radial scaling factor, $\alpha^\star(u)$, often admits an efficient closed-form solution. This contrasts sharply with standard layers based on orthogonal projection ($\min_{v \in \mathcal{C}} \|u - v\|^2$). While orthogonal projection onto a single ball or simplex admits efficient direct algorithms, projection onto polyhedral sets formed by intersecting linear constraints generally requires solving a constrained quadratic program (QP) for every sample in the batch. These QPs typically lack closed-form solutions, necessitating iterative numerical solvers---such as Dykstra's algorithm for intersections of convex sets \citep{dykstra1983algorithm, boyle1986} or ADMM for splitting complex composite constraints \citep{boyd2010}---that significantly slow down the forward pass. For constraint families admitting closed-form radial scaling, our method circumvents this bottleneck by replacing the iterative QP with a direct, vectorized calculation.

\textbf{Choice of radial contraction $r$.}
Recall from Definition~\ref{def:soft-projection} that the radial contraction $r$ maps the squared distance $\rho \coloneqq \|u-u_0\|^2$ to a scaling factor in $[0, 1)$. We parameterize this contraction using smooth sigmoidal mappings satisfying Assumption~\ref{ass:radial}. Fixing parameters $\varepsilon\in(0,1)$ and $\lambda>0$, we consider three families: (Rational) $r(\rho) = \varepsilon + (1-\varepsilon)\rho/(\rho+\lambda)$, (Exponential) $r(\rho) = \varepsilon + (1-\varepsilon)(1 - e^{-\rho/\lambda})$, and (Hyperbolic) $r(\rho) = \varepsilon + (1-\varepsilon)\tanh(\rho/\lambda)$. While all three forms guarantee strict interior feasibility, they differ in their asymptotic saturation rate (see Appendix~\ref{subsec:radial-contraction}). Inside $\mathcal C$, the radial derivative $r(\rho)+2\rho\,r'(\rho)$ remains bounded away from zero, ensuring that the Jacobian retains full rank, including at the anchor $u_0$. Along rays with finite boundary distance, the radial derivative tends to zero as the unconstrained input moves farther beyond the feasible boundary.

\textbf{Computing the radial map.} The evaluation of $p(u)$ reduces to computing the scaling factor $\alpha^\star(u)$ of the hard radial projection $q(u)=u_0+\alpha^\star(u)(u-u_0)$. By construction, $\alpha^\star(u)=1$ for feasible inputs, while for infeasible inputs $\alpha^\star(u)$ determines the intersection of the ray from $u_0$ through $u$ with the boundary $\partial\mathcal C$.

\textit{Case 1: Polyhedra (Linear constraints).}
For sets defined by linear inequalities $Ax \leq b$ (\textit{e.g.}, the simplex or capped simplex), the hard radial scaling factor is obtained exactly without iteration. Since $u_0$ is strictly interior ($b-Au_0>0$), we obtain
\[
\alpha^\star(u)
=
\min\left\{
1,\,
\min_{i:\,a_i^\top(u-u_0)>0}
\frac{b_i-a_i^\top u_0}{a_i^\top(u-u_0)}
\right\},
\]
where the inner minimum is $+\infty$ if the index set is empty. The computation is fully parallelizable and requires $\mathcal O(mn)$ operations for a dense constraint matrix $A\in\mathbb R^{m\times n}$.\newline
\textit{Case 2: Euclidean balls.} For spherical constraints $\|x - c\|_2 \leq R$, the positive root $t^\star$ of $\|(u_0 + t(u-u_0)) - c\|^2 = R^2$ admits an analytic solution, yielding $\alpha^\star(u)=\min\{1,t^\star\}$ for $u\neq u_0$, with $\alpha^\star(u_0)=1$.\newline
\textit{Case 3: General convex sets.}
For a general convex set $\{x\in\mathbb R^n:h(x)\leq0\}$ with convex $h$, computing $\alpha^\star(u)$ reduces to a one-dimensional search along the ray, with $\phi(t)=h(u_0+t(u-u_0))$. If $u\in\mathcal C$, then $\alpha^\star(u)=1$. Otherwise, convexity ensures that the feasible values of $t\in[0,1]$ form an interval containing zero, whose endpoint $\alpha^\star(u)$ can be determined by bisection.

\textbf{Choice of the interior point $u_0$.} Any $u_0\in\operatorname{Int}(\mathcal C)$ is sufficient for the construction, but conditioning depends on the margin $\delta\coloneqq\operatorname{dist}(u_0,\partial\mathcal C)$. Smaller margins weaken the Lipschitz bounds for the radial map, motivating the use of high-slack anchors in practice, such as barycenters for simplices, normalized cap vectors for capped simplices, Chebyshev centers for general polytopes, or analytic/barrier centers when available.\footnote{We isolate anchor sensitivity and boundary optima in Appendix~\ref{subsec:controlled-polytope-benchmarks}.}

\textbf{Baselines.} To provide insights with different neural network architectures, we consider for $g_\theta(z)$ a multi-layer perceptron (MLP) and a Long Short-Term Memory (LSTM) \citep{hochreiter1997long} architecture. To benchmark the efficacy of the soft-radial projection, we compare against state-of-the-art constraint enforcement layers ($\operatorname{Proj}(g_\theta(z))$) across four categories:
\begin{itemize}[leftmargin=*, nosep]
    \item \textit{Softmax (Simplex):} For standard simplex constraints, we utilize the Softmax function with temperature scaling $\tau$. While computationally cheap, it does not generalize to arbitrary convex sets.
    \item \textit{Orthogonal projection:} We compare against the standard Euclidean projection. As discussed in Section~\ref{sec:intro}, this method is idempotent---preserving the landscape for already-feasible points---but suffers from gradient saturation when inputs lie outside $\mathcal{C}$.
    \item \textit{Gauge/radial projection:} We compare against \emph{RAYEN}~\citep{tordesillas2023rayen}, a hard gauge-based layer that maps infeasible outputs along rays to $\partial\mathcal C$. Unlike SRP, it lacks a soft interior contraction and is many-to-one outside the feasible set.
    \item \textit{Optimization-based layers:} We evaluate adapted versions of \emph{Deep Constraint Completion and Correction} (DC3) \citep{dontidc3}, which enforces feasibility via gradient-based corrections, and \emph{HardNet} \citep{min2024hardnet}, a recent architecture specialized for affine constraints.\footnote{App.~\ref{app:competitor-methods} describes the implementation of the respective models for the capped simplex.}
\end{itemize}
\vspace{-0.5\baselineskip}

\subsection{Portfolio Optimization}
\label{subsec:portfolio-optimization}

We consider a financial market with $N$ assets over a finite horizon $T$. Let $(\Omega, \mathcal{F}, \mathbb{P})$ be a probability space equipped with a filtration $(\mathcal{F}_t)_{t \geq 0}$. At each time step $t$, the agent observes signals $z_t$--comprising raw returns augmented with rolling volatility and correlation to the market proxy computed over lookback $H$---and selects a portfolio weight vector \( w_t \in \mathcal{C} \subset \mathbb{R}^N \). Between steps $t-1$ and $t$, simple asset returns $y_t \in \mathbb{R}^N$ induce a passive drift, evolving weights to $w_{t-1}^{+} = (\operatorname{diag}(\mathds{1}+y_t)\,w_{t-1}) / ((\mathds{1}+y_t)^\top w_{t-1})$. We account for the transaction cost incurred when selecting $w_t$ at time $t$ and the portfolio return realized over the subsequent period. The objective is to maximize the Sharpe ratio of the resulting net returns $R^{\text{net}}_{t+1}$
\begin{equation}
    \label{eq:sharpe_opt}
    \displaystyle\max_{\theta}\; \frac{\mathbb{E}[R^{\text{net}}_{t+1}]}{\sqrt{\operatorname{Var}(R^{\text{net}}_{t+1})}} \quad \text{s.t.} \quad R^{\text{net}}_{t+1} = w_{t}^{\top}y_{t+1} - \frac{\gamma}{2}\|w_{t} - w_{t-1}^{+}\|_{1}, \;\; w_{t} = \operatorname{Proj}_{\mathcal{C}}(g_{\theta}(z_{t})) \;\; \forall t.
\end{equation}
Here, $\gamma$ is the transaction cost parameter and the $L_1$ norm captures turnover costs.\footnote{During training, we approximate the $L_1$ norm using the smooth Pseudo-Huber loss; evaluation uses exact turnover.} We employ the \textit{component-wise capped simplex}, which generalizes the standard simplex by enforcing asset-specific capacities $c \in (N^{-1}, 1]^N$ to control risk concentration $\mathcal{C} \;=\; \left\{ w \in \mathbb{R}^N_+\;\middle|\; \mathds{1}^\top w = 1, \; w \leq c \right\}$.

\begin{wraptable}{r}{0.78\textwidth}
\vspace{-1em}
\centering
\small
\caption{Portfolio results for the \textit{Global} dataset, aggregated over five seeds (mean $\pm$ std) with $N=11$ assets, per-asset capacity $c=0.25$, and transaction cost of 10 bps. DC3 and HardNet include evaluation-time feasibility correction; see App.~\ref{app:competitor-methods}.}
\label{tab:portfolio_results}
\begin{tabular}{clcccc}
\toprule
Set & Method &  SR (Net)$\uparrow$ & Turnover$\downarrow$ & Training [s]$\downarrow$ & Inference [ms]$\downarrow$\\
\midrule
$\Delta$ & Softmax & $0.65\,(0.09)$ & $0.11\,(0.02)$ & $11.16\,(0.19)$ & $0.34\,(0.08)$ \\
\midrule
\multirow{5}{*}{$\mathcal{C}$}
& O-Proj & $0.44\,(0.13)$ & $0.16\,(0.02)$ & $46.29\,(0.97)$ & $5.81\,(0.19)$ \\
& DC3 & $0.25\,(0.18)$ & $0.23\,(0.01)$ & $23.52\,(0.41)$ & $8.44\,(0.36)$ \\
& HardNet & $0.46\,(0.21)$ & $0.16\,(0.04)$ & $17.27\,(0.92)$ & $6.87\,(0.10)$ \\
& RAYEN & $0.60\,(0.06)$ & $0.11\,(0.01)$ & $14.09\,(0.25)$ & $0.68\,(0.01)$ \\
& SRP & $0.70\,(0.06)$ & $0.04\,(0.02)$ & $14.73\,(0.14)$ & $0.76\,(0.01)$\\
\bottomrule
\end{tabular}
\end{wraptable}

\textbf{Results analysis.} Table~\ref{tab:portfolio_results} evaluates the net Sharpe ratio and turnover of an LSTM policy trained end-to-end under various constraint enforcement methods. On the standard simplex ($\Delta$), the Softmax baseline achieves a Sharpe ratio of $0.65$ without the additional per-asset capacity constraints. Under the capped simplex ($\mathcal{C}$), O-Proj achieves a Sharpe ratio of $0.44$, while the adapted HardNet and DC3 baselines attain $0.46$ and $0.25$, respectively. Conversely, our soft-radial projection achieves the highest mean Sharpe ratio ($0.70$), with low turnover ($0.04$) and substantially lower inference cost than the correction-based methods. Notably, the hard gauge-based RAYEN matches our computational efficiency by design but yields a lower Sharpe ratio ($0.60$) with higher turnover. Overall, soft-radial projection combines the low latency of ray-based enforcement with the best observed solution quality on this task.

\vspace{-0.5\baselineskip} 
\subsection{Ride-Sharing Dispatch}
\label{subsec:ridesharing}

To demonstrate generalizability, we apply our framework to ride-sharing dispatch with a \emph{time-varying budget}. Unlike the fixed portfolio sum, the available fleet size $S_t\in\mathbb{N}_+$ fluctuates over time, causing the feasible set $\mathcal{C}_t$ to expand and contract. In practice, such constraints enforce per-zone capacity limits (\textit{e.g.}, congestion or emissions control) \cite{alonso2017demand, zardini2022analysis}. At each time step $t$, the agent observes a demand history window $z_t$ (augmented with sine/cosine encodings for day and week) and a current supply proxy $S_t$. We seek an allocation $a_t\in\mathbb{R}^N_+$ over $N$ zones to maximize the \textit{served rate} for the upcoming interval. We optimize the empirical average:
\begin{equation}
    \label{eq:taxi_opt}
    \displaystyle \max_\theta \; \frac{1}{T} \sum_{t=1}^{T} \left[ \frac{\sum_{i=1}^N \operatorname{min}_\tau(a_{i,t}, d_{i,t})}{\sum_{i=1}^N d_{i,t}} \right] \qquad \text{s.t.} \qquad a_t = \operatorname{Proj}_{\mathcal{C}_t}(g_\theta(z_t)) \;\; \forall t.
\end{equation}
where $d_{t}$ is the realized demand in the interval following the observation window. We replace the hard minimum with the smooth SoftMin, $\operatorname{min}_\tau(x, y) = -\tau \log(e^{-x/\tau} + e^{-y/\tau})$, to maintain non-zero gradients even when supply exceeds demand ($a_{i,t} > d_{i,t}$). The constraint set is a scaled capped simplex
$\mathcal{C}_t \!= \!\left\{a \in \mathbb{R}^N_+ \;\middle|\; \mathds{1}^\top a = S_t, \; a \leq \kappa S_t \right\}$. We implement $\operatorname{Proj}$ via a \textit{scaling transformation} (homothety): the network output is normalized by $S_t$, projected onto the canonical unit capped simplex, and subsequently rescaled by $S_t$. This effectively maps the dynamic constraint geometry to a fixed reference polytope.

\begin{wraptable}{r}{0.67\textwidth}
\vspace{-1em}
\caption{Ride-sharing dispatch results, aggregated over five random seeds (mean $\pm$ std), for the top $N=150$ activity zones and zone capacity $\kappa=0.1$; see App.~\ref{subsec:extended-analysis-ride-sharing-dispatch}.}
\centering
\small
\begin{tabular}{clccc}
\toprule
Set & Method & Served Rate$\uparrow$ & Training [s]$\downarrow$ & Inference [ms]$\downarrow$\\
\midrule
$\Delta_t$ & Softmax & $0.84\,(0.00)$ & $12.53\,(0.72)$ & $0.34\,(0.06)$ \\
\midrule
\multirow{5}{*}{$\mathcal{C}_t$} & O-Proj & $0.72\,(0.01)$ & $80.25\,(4.11)$ & $10.65\,(0.33)$ \\
& DC3 & $0.84\,(0.00)$ & $61.39\,(1.30)$ & $17.15\,(2.04)$ \\
& HardNet & $0.84\,(0.00)$ & $44.88\,(0.69)$ & $13.57\,(0.74)$ \\
& RAYEN & $0.84\,(0.00)$ & $17.92\,(0.65)$ & $0.79\,(0.07)$ \\
& SRP & $0.84\,(0.00)$ & $15.52\,(0.30)$ & $0.89\,(0.01)$ \\
\bottomrule
\end{tabular}
\label{tab:taxi-results}
\vspace{-2em}
\end{wraptable}

\textbf{Results analysis.} Table~\ref{tab:taxi-results} evaluates the served rate (demand fulfillment) of an MLP policy. Unlike the highly stochastic portfolio domain, the ride-sharing environment exhibits more regular demand patterns, allowing the unconstrained Softmax baseline to achieve a served rate of $0.84$. Under dynamic capacity constraints ($\mathcal{C}_t$), DC3, HardNet, RAYEN, and our soft-radial projection all match this performance ($0.84$). This demonstrates that soft-radial projection achieves the best observed solution quality while enforcing feasibility at a fraction of the computational cost of correction-based methods. Conversely, orthogonal projection underperforms ($0.72$), consistent with the loss of gradient information along directions suppressed by active constraints. Finally, while hard gauge-based methods such as RAYEN perform equally well in this structured setting, our soft-radial mapping additionally guarantees injectivity and a full-rank Jacobian almost everywhere, avoiding the many-to-one collapse characteristic of boundary-valued radial projections.
\vspace{-0.5\baselineskip} 
\section{Conclusion}
\label{sec:conclusion}
This paper introduced \textit{Soft-Radial Projection} to address the gradient saturation associated with boundary-valued projection layers. By constructing a radial homeomorphism to the interior of a convex set, we guarantee strict feasibility and preserve model expressivity (Theorem~\ref{thm:ua}) without relying on iterative optimization for common constraint families. Empirically, the method achieves competitive solution quality at low computational cost, with particularly strong results in portfolio optimization. These findings establish Soft-Radial Projection as an efficient alternative to conventional projection layers for enforcing hard constraints in end-to-end learning.

\textbf{Limitations.} While Soft-Radial Projection applies to general convex feasible sets, input-dependent constraints introduce additional computational overhead when the interior anchor and hard radial projection must be computed for each instance. Furthermore, boundary optima are approached asymptotically, introducing a trade-off between approximation accuracy and gradient conditioning. Since Soft-Radial Projection is not idempotent, the method is designed for end-to-end training rather than post-hoc correction of pretrained predictions. Developing efficient implementations for input-dependent constraints and further understanding this trade-off offer promising directions for future research.

\subsection*{Acknowledgments}

We thank Yves Rychener for his valuable insights and contributions to the initial conceptualization of this work. This work was supported as part of NCCR Automation, a National Centre of Competence in Research, funded by the Swiss National Science Foundation (grant number 51NF40\_225155).

\bibliographystyle{unsrtnat}
\bibliography{references}


\newpage

\appendix

\section{Related Work}
\label{sec:related-work}

Enforcing hard constraints in neural networks is a fundamental challenge in safety-critical learning, requiring architectures that guarantee output feasibility without compromising optimization stability. Optimization over the space of constrained functions is generally intractable. However, we can simplify the task-driven formulation~\eqref{eq:task} by leveraging the \emph{Interchangeability Principle}~\citep[Theorem~14.60]{rockafellar2009variational}. Applied to this setting, the principle establishes that minimizing the expected cost is equivalent to minimizing the cost pointwise for almost every input $z$, formally $\inf_{\pi \in \Pi,\, \pi(z) \in \mathcal{C}} \mathbb{E}_{z\sim\mathbb{P}_Z}[c(z, \pi(z))] = \mathbb{E}_{z\sim\mathbb{P}_Z}\left[\inf_{x \in \mathcal{C}} c(z, x)\right]$. This equivalence motivates solving feasibility pointwise by reparameterizing the decision policy: we optimize an unconstrained candidate function $g:\mathcal{Z} \to \mathbb{R}^n$ composed with a fixed projection operator, such that $\pi(z)=\operatorname{Proj}(g(z))$.

The literature can be categorized by the mechanism of enforcement---ranging from soft relaxations and iterative solvers to rigorous geometric constructions. Table~\ref{tab:comparison-methods} summarizes the properties of representative state-of-the-art approaches.

Broadly, constraint layers are distinguished by their dependency on the input context. \emph{Input-independent constraints} involve a feasible set $\mathcal{C}$ that is fixed across all samples (\textit{e.g.}, static actuator limits or portfolio weights), whereas \emph{input-dependent constraints} involve a dynamic geometry $\mathcal{C}(z)$ conditioned on the input. Our work primarily targets the former, establishing a rigorous geometric framework for fixed convex sets. While strict constraint satisfaction is essential, the quality of the learning signal---specifically, the ability of backpropagation to provide informative gradients in high-noise regimes---remains under-explored. Existing approaches prioritize inference speed and feasibility during training and execution. However, convergence behavior toward the true solution and the ultimate optimality gap are less understood, particularly when the enforcement layer significantly alters the unconstrained network output. While this review is not exhaustive---with ongoing developments ranging from accelerated orthogonal projections~\citep{grontas2026pinet} to the use of decision rules for constraint enforcement~\citep{constante2025enforcing}---different applications and noise profiles often demand alternatives to standard orthogonal projection to maintain superior optimization properties.

\begin{table*}[h]
\centering
\caption{\textbf{Comparison with state-of-the-art constraint layers.}
We position Soft-Radial Projection among constructive geometric layers that guarantee feasibility without high-dimensional iterative solvers.}
\label{tab:comparison-methods}
\resizebox{\textwidth}{!}{%
\begin{tabular}{l l l l l l c}
\toprule
\textbf{Method} & \textbf{Constraint Scope} & \textbf{Input Dep.} & \textbf{Satisfaction} & \textbf{Computation} & \textbf{Gradient Landscape} & \textbf{Univ. Approx.} \\
\midrule
Soft-Penalty & Any & Yes & No (Penalty) & Closed-Form & Smooth (infeasible) & Yes \\
DC3 \cite{dontidc3} & General & Yes & Approx.$^\star$ & Iterative (GD) & Step-size sensitive & Unknown \\
RAYEN \cite{tordesillas2023rayen} & Convex (Lin/QC/SOC/LMI) & Limited & Strict & Closed-Form & Clipped / Flat & Unknown \\
HardNet-Aff \cite{min2024hardnet} & Affine & Yes & Strict & Closed-Form & Saturated (Boundary) & Yes \\
HP \cite{liang2024} & Ball-Homeomorphic & Yes & Strict & Bisection & Iterative search & Yes \\
HoP (Polar) \cite{deng2025hop} & Star-Convex & Yes & Strict (Int.) & Closed-Form$^\dagger$ & Polar stagnation & Unknown \\
RP \cite{konstantinov2023} & Convex (Lin/Quad) & Limited & Strict (Int.) & Closed-Form & Smooth (Saturating) & Unknown \\
\midrule
\rowcolor{gray!10}
\textbf{SRP (Ours)} & \textbf{Convex} & \textbf{No$^\ddagger$} & \textbf{Strict (Int.)} & \textbf{C.-F. / 1D Root} & \textbf{Full-rank Jacobian a.e.} & \textbf{Yes} \\
\bottomrule
\end{tabular}%
}
\vspace{2pt}
\footnotesize{
$^\star$Asymptotic (linear constraints); heuristic/approximate otherwise; truncated correction in practice. \\
$^\dagger$HoP requires a per-sample interior anchor and a ray--boundary intersection query. \\
$^\ddagger$Current focus is on fixed sets $\mathcal{C}$, though formulation supports dynamic anchors.
}
\end{table*}

\subsection{Approximation and Optimization Approaches}

\textbf{Soft penalties and activations.}
The most straightforward approach, favored by many industry practitioners, relaxes strict constraints into the training objective. Penalty-based methods augment the loss function with a regularization term proportional to the magnitude of constraint violations \cite{marquez2017imposing}. While this preserves the smoothness of the optimization landscape, it fundamentally fails to guarantee feasibility during inference, particularly in high-dimensional spaces where the volume of the feasible set relative to the ambient space can become negligible. Similarly, standard activation functions (\textit{e.g.}, Sigmoid, Softmax) offer computationally efficient bounds for elementary box or simplex constraints but lack the topological expressivity to model complex coupled constraints or polyhedral boundaries.

\textbf{Differentiable optimization and iterative correction.} To ensure strict satisfaction, a second family of methods embeds optimization solvers directly into the network architecture. Differentiable optimization layers, such as OptNet \cite{amos2017optnet} and CVXPYLayers \cite{agrawal2019differentiable}, compute the orthogonal projection of a prediction onto $\mathcal{C}$ by solving a convex program during the forward pass. DC3 \cite{dontidc3} instead utilizes a \emph{completion--correction} scheme, where equality constraints are satisfied via variable completion and inequality violations are reduced through iterative gradient-based correction. While mathematically general, these approaches introduce significant computational overhead. Furthermore, backpropagating through these layers typically relies on implicit differentiation, which assumes the inner solver has reached a stationary point. When truncated for real-time efficiency (\textit{e.g.}, fixed-iteration limits), this assumption is violated, leading to conflict-prone gradient updates where the descent direction of the task loss contradicts the incomplete constraint correction.

\subsection{Constructive Geometric Layers}

To circumvent the computational cost of iterative solvers, recent research focuses on ``Feasibility by Construction'' via closed-form surjective mappings $\operatorname{Proj}: \mathbb{R}^n \to \mathcal{C}$. We classify these geometric approaches into three distinct families:

\textbf{Gauge and boundary projections.}
RAYEN \cite{tordesillas2023rayen} enforces convex constraints by selecting a fixed interior point and \emph{rescaling} a predicted direction so that the endpoint remains feasible; the mapping is identity when the step stays inside the set, and otherwise returns the first ray--boundary intersection. HardNet-Aff \cite{min2024hardnet} targets \emph{input-dependent affine} inequalities via a closed-form pseudoinverse-based correction that clamps violated constraints to their bounds while moving in directions parallel to already-satisfied constraint boundaries. While both layers are computationally efficient, their piecewise ray/clamping structure can suppress learning signals when constraints are active (\textit{e.g}., RAYEN becomes insensitive to step magnitude after rescaling, and HardNet-Aff can attenuate gradient components in constraint-normal directions). This suppression can bias training toward boundary-feasible outputs and reduce interior exploration.

\textbf{Topological diffeomorphisms.}
To utilize the set's interior, topological methods reshape the output space. Homeomorphic Projection (HP) \cite{liang2024} learns a bi-Lipschitz \emph{invertible neural network} (INN) that approximates a minimum-distortion homeomorphism between the constraint set and a unit ball. At inference, HP maps an infeasible prediction into the ball via the INN inverse and then performs a $k$-step bisection along the ray to the ball center, requiring repeated INN forward evaluations and feasibility checks; the resulting boundary-feasible point trades computation for accuracy through the bisection depth and depends on the learned mapping distortion.
Polar coordinate methods like HoP \cite{deng2025hop} attempt an analytic solution but introduce polar stagnation, singularities at the origin where angular components are undefined, resulting in numerical instability near the geometric center.

\textbf{Radial interior mappings.} 
Closely related to our work, \citet{konstantinov2023} propose a ``ray-projection'' layer that constructs feasible points by scaling a direction vector $r$ from an interior anchor. While effective, this requires a specialized architecture that learns separate polar coordinates ($g(r, s)$) rather than a direct map of the ambient Euclidean space. Furthermore, these constructions typically do not furnish a global bijection between $\mathbb{R}^n$ and $\operatorname{Int}(\mathcal{C})$, potentially resulting in saturating radial behavior due to the choice of bounded scaling functions (\textit{e.g.}, sigmoids).

\textbf{Our contribution: Soft-radial projection.} We propose a drop-in reparameterization layer for fixed closed convex sets with a non-empty interior. Our mapping acts directly on the unconstrained ambient prediction $u \in \mathbb{R}^n$ as a unified radial homeomorphism $\mathbb{R}^n \to \operatorname{Int}(\mathcal{C})$. For standard convex geometries---such as balls, ellipsoids, and boxes---the mapping admits a closed-form solution. For general convex sets, we utilize an efficient monotonic 1D root-finding procedure. Crucially, we ensure the mapping is differentiable with an invertible Jacobian almost everywhere, preventing the rank-collapse characteristic of boundary projections. This avoids exact projection-induced rank collapse, although the radial derivative diminishes as the unconstrained input moves farther beyond the feasible boundary. The construction also preserves the universal approximation property of the base network (Theorem~\ref{thm:ua}).

\newpage

\section{Proofs}
\label{app:proofs}

In the following sections, we provide complete proofs for the theoretical results. To facilitate readability, we restate each claim prior to its proof.

\subsection{Proofs of Section \ref{sec:soft-projection}}

We begin by establishing the regularity of the individual projection components.

\RayRegularity*

\begin{proof}
\emph{Ray feasibility and intersection.} Fix $u\in\mathbb R^n$ and consider scalars $\alpha\in[0,1]$ such that $\alpha u\in\mathcal C$. The map $\alpha\mapsto \alpha u$ is continuous and $\mathcal C$ is closed, so the set of feasible $\alpha\in[0,1]$ is closed in $[0,1]$ and nonempty (since $\alpha=0$ is always feasible). Convexity of $\mathcal C$ implies this set is convex in $\mathbb R$, and any nonempty closed convex subset of $\mathbb R$ is an interval; hence it is of the form $[0,\alpha^\star(u)]$ for a unique $\alpha^\star(u)\in[0,1]$.

\emph{Global Lipschitz continuity of $\alpha$.} Let $\gamma_{\mathcal{C}}:\mathbb R^n\to[0,\infty)$ denote the Minkowski functional (gauge) of $\mathcal C$,
\[
\gamma_{\mathcal{C}}(x) \coloneqq \inf\{t>0 : x\in t\mathcal C\}.
\]
Standard properties of the gauge (see, \textit{e.g.}, \citet{rockafellar1970convex}, Thm.~14.5 and Chap.~13) imply that $\gamma_{\mathcal{C}}$ is sublinear (convex and positively homogeneous) and finite on $\mathbb R^n$ when $0\in\operatorname{Int}(\mathcal C)$. Because $0\in\operatorname{Int}(\mathcal C)$, there exists $\delta>0$ with $B(0,\delta)\subset\mathcal C$, which yields $\gamma_{\mathcal{C}}(x)\le \|x\|/\delta$ for all $x$ and, by subadditivity of $\gamma_{\mathcal{C}}$,
\[
\gamma_{\mathcal{C}}(x)-\gamma_{\mathcal{C}}(y) \le \gamma_{\mathcal{C}}(x-y) \le \frac{\|x-y\|}{\delta}
\quad\forall\;x,y\in\mathbb R^n,
\] 
and 
\[
\gamma_{\mathcal{C}}(y)-\gamma_{\mathcal{C}}(x) \le \gamma_{\mathcal{C}}(y-x) \le \frac{\|y-x\|}{\delta}
\quad\forall\;x,y\in\mathbb R^n.
\]
Taking the maximum of the two inequalities shows
\[
|\gamma_{\mathcal{C}}(x)-\gamma_{\mathcal{C}}(y)| \le \frac{\|x-y\|}{\delta}
\quad\forall\;x,y\in\mathbb R^n,
\]
so $\gamma_{\mathcal{C}}$ is Lipschitz with constant $1/\delta$.

For $u\neq 0$, the feasibility of $\alpha u\in\mathcal C$ is equivalent to $\alpha\le1/\gamma_{\mathcal{C}}(u)$. Therefore, the feasible $\alpha\in[0,1]$ form the interval $[0,\min\{1,1/\gamma_{\mathcal{C}}(u)\}]$, implying that $\alpha^\star(u) = \min\bigl(1,1/\gamma_{\mathcal{C}}(u)\bigr)$, for $u\neq 0$, while for $u=0$ we clearly have $\alpha^\star(u)=1$. Define a helper function $\psi:[0,\infty)\to(0,1]$ by $\psi(t) \coloneqq \min(1,1/t)$ for $t>0$ and $\psi(0) \coloneqq 1$. The function $\psi$ is $1$–Lipschitz on $[0,\infty)$, so for all $u,v\in\mathbb R^n$,
\begin{align*}
|\alpha^\star(u)-\alpha^\star(v)| &= |\psi(\gamma_{\mathcal{C}}(u))-\psi(\gamma_{\mathcal{C}}(v))|\\
&\le |\gamma_{\mathcal{C}}(u)-\gamma_{\mathcal{C}}(v)|\le \frac{\|u-v\|}{\delta}.
\end{align*}
Thus $\alpha^\star$ is globally Lipschitz with constant $1/\delta$.

\emph{Local Lipschitz continuity of $q$.} We now bound the variation of $q(u)=\alpha^\star(u)u$ on bounded sets. Let $R>0$ and assume $\|v\|\le R$ and $\|u\|\le R$.
Then
\begin{align*}
\|q(u)-q(v)\| &= \|\alpha^\star(u)u-\alpha^\star(v)v\| \\
&\le \|\alpha^\star(u)(u-v)\| + \|v(\alpha^\star(u)-\alpha^\star(v))\| \\
&\le \|u-v\| + \|v\|\cdot|\alpha^\star(u)-\alpha^\star(v)| \\
&\le \|u-v\| + \frac{\|v\|}{\delta}\,\|u-v\| \\
&\le \Bigl(1+\frac{R}{\delta}\Bigr)\|u-v\|.
\end{align*}
Hence $q$ is $(1+R/\delta)$–Lipschitz on the ball $\{u:\|u\|\le R\}$, which implies that $q$ is locally Lipschitz on $\mathbb R^n$, completing the proof.
\end{proof}

Next, we analyze the structural properties of the constructed projection map.

\RegComponents*

\begin{proof}
\emph{Continuity of $\bar t$.} We relate $\bar t(v)$ to the Minkowski functional (gauge) $\gamma_{\mathcal{C}}$. By definition, $\bar t(v) = 1/\gamma_{\mathcal{C}}(v)$ (with the convention $1/0 = \infty$). Since $\gamma_{\mathcal{C}}$ is continuous on $\mathbb R^n$ (Lemma~\ref{lem:ray-regularity}) and the inversion map is continuous on $[0, \infty)$ into $(0, \infty]$, the composition $v \mapsto \bar t(v)$ is continuous on the unit sphere.

\emph{Properties of $\psi_v$.} For $t \le \bar t(v)$, we have $\psi_v(t) = r(t^2)t$. Strict monotonicity on this interval was established immediately following Assumption~\ref{ass:radial}. For $t > \bar t(v)$ (which implies $\bar t(v) < \infty$), we have $\psi_v(t) = r(t^2)\bar t(v)$, which is strictly increasing simply because $r$ is increasing. Continuity holds at $t=\bar t(v)$ as both expressions coincide. Finally, since $\psi_v(0)=0$ and $\lim_{t\to\infty} \psi_v(t) =\bar t(v)$, the range is exactly $[0, \bar t(v))$.
\end{proof}

Building on these established properties, we prove that $p$ is indeed a homeomorphism.

\Homeomorphism*

\begin{proof}

We first establish that $p$ is a bijection.

\emph{Injectivity.} Let $u_1, u_2 \in \mathbb{R}^n$ with $p(u_1) = p(u_2)$. Since $p$ preserves directions (\textit{i.e.}, $p(u) \in \mathbb{R}_+ u$), $u_1$ and $u_2$ must lie on the same ray, so $u_i = t_i v$ for some unit vector $v$ and $t_i \ge 0$. The scalar map $\psi_v$ from Definition~\ref{def:scalar-map} satisfies $p(t_i v) = \psi_v(t_i)v$. By Lemma~\ref{lem:reg-comp}, $\psi_v$ is strictly increasing, so $\psi_v(t_1) = \psi_v(t_2)$ implies $t_1 = t_2$, and hence $u_1 = u_2$.

\emph{Surjectivity.} Let $x \in \operatorname{Int}(\mathcal{C})$. If $x=0$, $x=p(0)$. If $x \neq 0$, set $v = x/\|x\|$ and $s = \|x\|$. Since $x \in \operatorname{Int}(\mathcal{C})$, we strictly have $s < \bar{t}(v)$. By Lemma~\ref{lem:reg-comp}, the continuous scalar map $\psi_v$ has range $[0, \bar{t}(v))$. By the Intermediate Value Theorem, there exists a unique $t$ such that $\psi_v(t) = s$, yielding $p(tv) = x$. Moreover, $p$ is continuous on $\mathbb{R}^n$ as it is composed of the continuous scaling function $r$ (Assumption~\ref{ass:radial}) and the continuous hard radial projection $q$ (Lemma~\ref{lem:ray-regularity}). Thus, $p$ is a continuous bijection from $\mathbb{R}^n$ onto $\operatorname{Int}(\mathcal{C})$.

\emph{Continuity of Inverse.} Since $p$ is a continuous injection from $\mathbb{R}^n$ into an open subset of $\mathbb{R}^n$, the Invariance of Domain theorem \cite{Brouwer1912} guarantees that $p$ is an open map. Because $p$ is a continuous bijection that maps open sets to open sets, it is a homeomorphism between $\mathbb{R}^n$ and its image $\operatorname{Int}(\mathcal{C})$.
\end{proof}

Finally, we analyze the differentiability of $p$, establishing that it is a diffeomorphism almost everywhere.

\DiffInvert*

\begin{proof}
By Lemma~\ref{lem:ray-regularity}, $q$ is locally Lipschitz and therefore differentiable almost everywhere by Rademacher's theorem. Since $r$ is $C^1$ and strictly positive, the identities $p(u)=r(\|u\|^2)q(u)$ and $q(u)=p(u)/r(\|u\|^2)$ show that $p$ and $q$ have the same domain of differentiability. Whenever $J_p(u)$ exists, the product rule gives
\begin{equation}
\label{eq:p-jacobian}
J_p(u) \;=\; r(\|u\|^2)\,J_q(u) \;+\; 2\,r'(\|u\|^2)\,q(u)\,u^\top.
\end{equation}
At a boundary point $u=t_0v$ with $\|v\|=1$ and $t_0=\bar t(v)$, the left and right derivatives of the radial profile $\psi_v$ differ by $r(t_0^2)>0$. Hence $p$ is not differentiable on $\partial\mathcal C$. We now examine the Jacobian in the interior and exterior of the feasible set.

\noindent\emph{Interior points.} For $u\in\operatorname{Int}(\mathcal C)$ we have $q(u)=u$ and $J_q(u)=I$, hence
\[J_p(u) = r(\|u\|^2)\,I + 2\,r'(\|u\|^2)\,u u^\top.\]
For $u\neq0$, this matrix has eigenvector $u/\|u\|$, with corresponding eigenvalue $r(\|u\|^2) + 2\,\|u\|^2\,r'(\|u\|^2)$ and all vectors orthogonal to $u$ are eigenvectors with eigenvalue $r(\|u\|^2)$. Assumption~\ref{ass:radial} gives $r(\rho)>0$ and $r(\rho)+2\rho\,r'(\rho)>0$ for all $\rho\ge0$, so $J_p(u)\succ0$ is invertible for all $u\in\operatorname{Int}(\mathcal C)\setminus\{0\}$. At the origin, $q$ coincides with the identity in a neighborhood of $0$ (because $0\in\operatorname{Int}(\mathcal C)$), so $J_q(0)=I$ and the rank-one term vanishes, giving $J_p(0)=r(0)\,I$, which is invertible since $r(0)>0$.

\noindent\emph{Exterior points.} For $u\notin\mathcal C$, the gauge $\gamma_{\mathcal C}(u)>1$ and $q(u)=u/\gamma_{\mathcal C}(u)\in\partial\mathcal C$. Fix such a point where $p$ is differentiable. Then $q$ is differentiable, and the identity $\gamma_{\mathcal C}(u)=\|u\|/\|q(u)\|$, valid throughout the exterior, implies that the gauge is differentiable there as well. Write $\lambda\coloneqq\gamma_{\mathcal C}(u)>1$. Differentiating the exterior expression for $q$ yields
\[J_q(u) \;=\; \frac{1}{\lambda}I - \frac{1}{\lambda^2}\,u\,\nabla \gamma_{\mathcal{C}}(u)^\top.\]
We establish that $J_p(u)$ is invertible by showing its kernel is trivial. Let $h \in \mathbb{R}^n$ satisfy $J_p(u)h = 0$. We decompose $h$ with respect to the subspace $W = \{w \in \mathbb{R}^n \mid u^\top w = 0\}$. Any $h$ can be uniquely written as $h = h_W + \beta u$, where $h_W \in W$ and $\beta \in \mathbb{R}$.

First, we compute the product of $J_q(u)$ with the radial component. Since $\gamma_{\mathcal{C}}$ is positively homogeneous of degree 1 by definition, Euler's homogeneous function theorem implies $\nabla \gamma_{\mathcal{C}}(u)^\top u = \gamma_{\mathcal{C}}(u) = \lambda$. Using this relation, we find
\[
J_q(u)u \;=\; \frac{1}{\lambda}u - \frac{1}{\lambda^2}u \underbrace{(\nabla \gamma_{\mathcal{C}}(u)^\top u)}_{\lambda} \;=\; \frac{1}{\lambda}u - \frac{1}{\lambda}u \;=\; 0.
\]
For the orthogonal component $h_W$, we have
\[
J_q(u)h_W \;=\; \frac{1}{\lambda} h_W - \frac{\nabla \gamma_{\mathcal{C}}(u)^\top h_W}{\lambda^2} u.
\]
Substituting $h = h_W + \beta u$ and the expression for $J_p(u)$ \eqref{eq:p-jacobian} into the condition $J_p(u)h = 0$ yields
\[
r(\|u\|^2) J_q(u)h_W + 2 \beta \|u\|^2 r'(\|u\|^2) q(u) \;=\; 0.
\]
Substituting the expansion of $J_q(u)h_W$ allows us to group terms proportional to $u$. The equation becomes
\[
\frac{r(\|u\|^2)}{\lambda} h_W \;+\; \mu u = 0,
\]
where $\mu \in \mathbb{R}$ captures all radial coefficients:
\[
\mu = 2 \beta \|u\|^2 r'(\|u\|^2) \frac{1}{\lambda} - \frac{r(\|u\|^2) \nabla \gamma_{\mathcal{C}}(u)^\top h_W}{\lambda^2}.
\]
Projecting the vector equation onto the subspace $W$ eliminates $\mu u$, yielding
\[
\frac{r(\|u\|^2)}{\lambda} h_W = 0.\]
Since $r > 0$ and $\lambda > 1$, we must have $h_W = 0$. Consequently, the radial term must also vanish, so $\mu u = 0$, implying $\mu = 0$. With $h_W=0$, the expression for $\mu$ simplifies to
\[
2 \beta \|u\|^2 r'(\|u\|^2) \frac{1}{\lambda} = 0.
\]
Assumption~\ref{ass:radial} ensures $r'(\rho) > 0$ for $\rho > 0$, so we must have $\beta = 0$. Thus $h = 0$, proving that $J_p(u)$ is invertible. Combining the results for interior and exterior points, we conclude that $p$ is differentiable almost everywhere and $J_p(u)$ is invertible wherever it exists.
\end{proof}

\subsection{Proofs of Section \ref{sec:opt-guarantees}}
\label{app:proofs-opt-guarantees}

We begin by establishing the equivalence of the optimal values between the unconstrained composite optimization $f(u)$ and the constrained optimization objective.

\EquivOptVal*

\begin{proof}
By bijectivity of $p:\mathbb R^n\to\operatorname{Int}(\mathcal C)$,
\[
\inf_{u\in\mathbb R^n} f(u)=\inf_{x\in\operatorname{Int}(\mathcal C)} \ell(x).
\]
Since $u_0\in\operatorname{Int}(\mathcal C)$ and $\mathcal C$ is convex, for any $x\in\mathcal C$ and $\varepsilon\in(0,1]$,
$x_\varepsilon\coloneqq(1-\varepsilon)x+\varepsilon u_0\in\operatorname{Int}(\mathcal C)$ and $x_\varepsilon\to x$ as $\varepsilon\downarrow 0$.
By continuity, $\ell(x_\varepsilon)\to\ell(x)$, hence $\inf_{\operatorname{Int}(\mathcal C)}\ell\le \ell(x)$ for all $x\in\mathcal C$,
implying $\inf_{\operatorname{Int}(\mathcal C)}\ell\le \inf_{\mathcal C}\ell$.
Since $\operatorname{Int}(\mathcal C)\subseteq\mathcal C$, we also have $\inf_{\mathcal C}\ell\le \inf_{\operatorname{Int}(\mathcal C)}\ell$.
Thus the optimal values coincide.

Finally, $\arg\min f\neq\emptyset$ holds if and only if there exists $x^\star\in\operatorname{Int}(\mathcal C)$ with
$\ell(x^\star)=\inf_{x\in\mathcal C}\ell(x)$, in which case $u^\star=p^{-1}(x^\star)$ exists and satisfies $u^\star\in\arg\min f$.
Conversely, if $u^\star\in\arg\min f$, then $p(u^\star)\in\operatorname{Int}(\mathcal C)$ satisfies
$\ell(p(u^\star))=\inf_{x\in\mathcal C}\ell(x)$, \textit{i.e.}, $p(u^\star)\in\arg\min_{\mathcal C}\ell\cap\operatorname{Int}(\mathcal C)$.
\end{proof}

Next, we establish the equivalence of stationary points between the two objectives.

\CorrStatPoints*

\begin{proof}
The gradient identity follows from the chain rule, using that $p(u)\in\operatorname{Int}(\mathcal C)$ for all $u$. If $J_p(u)$ is invertible and $\nabla f(u)=J_p(u)^\top\nabla\ell(p(u))=0$, then $\nabla\ell(p(u))=0$. The converse is immediate.
\end{proof}

\paragraph{Value gap under bounded latent norm.}
As discussed in Section~\ref{sec:opt-guarantees}, when the optimizer lies on the boundary $\partial \mathcal C$, exact attainment under a strict-interior parameterization requires an unbounded latent norm. This is a structural property of any homeomorphism $p:\mathbb R^n\to \operatorname{Int}(\mathcal C)$: if a sequence $x_k\in \operatorname{Int}(\mathcal C)$ converges to $x\in \partial \mathcal C$, then necessarily $\|p^{-1}(x_k)\|\to\infty$. Otherwise, a bounded subsequence $u_{k_j} \coloneqq p^{-1}(x_{k_j})$ would have a convergent subsequence with limit $u\in\mathbb R^n$, and continuity of $p$ would imply $p(u)=x\in\partial\mathcal C$, contradicting $p(\mathbb R^n)=\operatorname{Int}(\mathcal C)$.

The practically relevant question is therefore the suboptimality induced by a bounded latent budget. Following the coordinate convention in Section~\ref{sec:soft-projection}, we work in translated coordinates and assume the anchor is at the origin. Hence, for every unit vector $v$, we have
\[
    p(tv)=\psi_v(t)v,
    \qquad
    \psi_v(t)=r(t^2)\min\{t,\bar t(v)\},
\]
where $\bar t(v)$ denotes the distance to the boundary along direction $v$. For a budget $M>0$, define the latent budget set $B_M \coloneqq \{u\in\mathbb R^n:\|u\|\le M\}$, the reachable feasible set $\mathcal C_M \coloneqq p(B_M)$, and the restricted optimal value
\[
    \ell_M^\star \coloneqq \inf_{x\in\mathcal C_M}\ell(x)
    =
    \inf_{\|u\|\le M}\ell(p(u)).
\]

\BudgetGap*

\begin{proof}
Let $x_M\coloneqq p(Mv^\star)=\psi_{v^\star}(M)\,v^\star$. Since $Mv^\star\in B_M$, we have $x_M\in\mathcal C_M$, so $\ell_M^\star\le \ell(x_M)$. Moreover, since $\mathcal C_M\subseteq\mathcal C$, we have $\ell^\star\le \ell_M^\star$. Using also $\ell^\star=\ell(x^\star)$, we obtain
\[
0\le \ell_M^\star-\ell^\star \le \ell(x_M)-\ell(x^\star)\le L\|x_M-x^\star\|.
\]
Because $x_M$ and $x^\star$ lie on the same ray, $\|x_M-x^\star\|=\bar t(v^\star)-\psi_{v^\star}(M)=\delta_M(v^\star)$. If $M\ge \bar t(v^\star)$, then $\psi_{v^\star}(M)=r(M^2)\bar t(v^\star)$, which gives $\delta_M(v^\star)=(1-r(M^2))\bar t(v^\star)$ and hence the stated specialization.
\end{proof}

The same argument also shows that no value gap is induced when the constrained problem admits an interior minimizer that is reachable under the budget. Indeed, if $x^\circ\in\arg\min_{x\in\mathcal C}\ell(x)\cap\operatorname{Int}(\mathcal C)$, then $u^\circ\coloneqq p^{-1}(x^\circ)$ has finite norm. Hence, for every $M\ge \|u^\circ\|$, we have $x^\circ\in\mathcal C_M$ and therefore $\ell_M^\star=\ell^\star$.

We then extend this analysis to show that stationary points of the composite objective correspond to global minimizers of the constrained problem.

\GlobOptInt*
\begin{proof}
Let $p$ be differentiable at $u$ with $\nabla f(u)=0$. Then $J_p(u)$ is invertible by Theorem~\ref{thm:differentiability}. By Proposition~\ref{prop:stationary}, $\nabla \ell(p(u))=0$. Since $p(u)\in\operatorname{Int}(\mathcal C)$ and $\ell$ is convex, $\nabla \ell(p(u))=0$ implies that $p(u)$ is a global minimizer of $\ell$ over $\operatorname{Int}(\mathcal C)$, and hence over $\mathcal C$ by continuity.
\end{proof}

We now investigate the convergence rate and demonstrate the absence of a global Polyak-\L{}ojasiewicz (PL) condition, which necessitates a more nuanced convergence analysis.

\AbsencePL*
\begin{proof}
The unique minimizer of $\ell$ over $\mathcal C$ is $x^\star=0$, and the corresponding optimal value $\ell^\star=0$. Fix a unit vector $v$ and consider $u(t)\coloneqq t v$ with $t\ge 0$. For $\mathcal{C}=B_1(0)$ the hard radial projection is $q(u)=u$ if $\|u\|\le 1$ and $q(u)=u/\|u\|$ otherwise; hence for $t\ge 1$ we have $q(u(t))=v$ and therefore 
\[
p(u(t)) \;=\; r(t^2)\,v,
\qquad
f(u(t)) \;=\; \|p(u(t))\|^2 \;=\; r(t^2)^2.
\]
Since $r(t^2)\to 1$ as $t\to\infty$, we obtain 
\[
f(u(t)) - f^\star \;=\; r(t^2)^2 \;\to\; 1 \quad (t\to\infty),
\]
so along this ray the suboptimality converges to a strictly positive constant. Next, compute the gradient for $t>1$. On $\{\|u\|>1\}$, the hard radial projection $q(u)=u/\|u\|$ has Jacobian 
\[
J_q(u) \;=\; \frac{1}{\|u\|}\Bigl(I-\frac{u u^\top}{\|u\|^2}\Bigr),
\]
so at $u(t)=t v$
\[
J_q(u(t)) \;=\; \frac{1}{t}\bigl(I-v v^\top\bigr).
\]
Substituting the expression for the Jacobian $J_p(u)$ given in \eqref{eq:p-jacobian}, we obtain for $t>1$
\[
J_p(u(t))=
\frac{r(t^2)}{t}(I-v v^\top) \;+\; 2t\,r'(t^2)\,v v^\top.
\]
By Proposition~\ref{prop:stationary}, the gradient is given by $\nabla f(u) = J_p(u)^\top \nabla \ell(p(u))$. Since $\ell(x)=\|x\|^2$ implies $\nabla \ell(x) = 2x$, and specifically $p(u(t)) = r(t^2)v$, we obtain
\[
\nabla f(u(t)) \;=\; J_p(u(t))^\top \bigl( 2 r(t^2) v \bigr) \;=\; 2 r(t^2) J_p(u(t)) v,
\]
where the last equality follows from the symmetry of $J_p(u(t))$. When applying the Jacobian matrix to the vector $v$, the tangential term vanishes because $(I - v v^\top)v = 0$. Only the radial term remains
\[
J_p(u(t)) v \;=\; \left( 2t r'(t^2) v v^\top \right) v \;=\; 2t r'(t^2) v.
\]
Substituting this back yields the final gradient norm
\[
\|\nabla f(u(t))\| \;=\; \| 2 r(t^2) \cdot 2t r'(t^2) v \| \;=\; 4\,t\,r(t^2)\,r'(t^2).
\]
It remains to show the existence of a sequence $t_k\to\infty$ such that $t_k r'(t_k^2)\to 0$. Since $\lim_{\rho\to\infty} r(\rho)=1$, we have $\int_0^\infty r'(\rho)\,d\rho = 1-r(0) < \infty$. We claim there exists a sequence $\rho_k\to\infty$ such that $\sqrt{\rho_k}\,r'(\rho_k)\to 0$. Indeed, if not, then there exist $\varepsilon>0$ and $\rho_0\geq0$ such that $\sqrt{\rho}\,r'(\rho)\ge \varepsilon$ for all $\rho\ge \rho_0$, which implies
\[
\int_{0}^\infty r'(\rho)\,d\rho \;\ge\; \varepsilon \int_{\rho_0}^\infty \frac{d\rho}{\sqrt{\rho}} \;=\; \infty,
\]
a contradiction. Setting $t_k\coloneqq\sqrt{\rho_k}$ yields $t_k\to\infty$ and $t_k r'(t_k^2)\to 0$. Along $u(t_k)=t_k v$ we therefore have $\|\nabla f(u(t_k))\|\to 0$ while $f(u(t_k))-f^\star \to 1$. Consequently, no $\mu>0$ can satisfy a global PL inequality $\|\nabla f(u)\|^2 \ge 2\mu\,(f(u)-f^\star)$ for all $u\in\mathbb R^n$.
\end{proof}

Although the global PL condition fails in general (Lemma~\ref{lem:pl}), this does not preclude standard first-order guarantees on compact iterate sets. The key distinction is between generic differentiability and smoothness. Since the soft-radial map is locally Lipschitz, exact nondifferentiability events are typically avoided under continuous sampling assumptions; this justifies ordinary backpropagation at generic iterates. However, almost-everywhere differentiability is weaker than the \(L\)-smoothness required by the classical smooth SGD rate. When iterates may cross boundary or active-face switching surfaces of the soft-radial map, the natural regularity class is nonsmooth. Proposition~\ref{prop:convergence} therefore states the compact-iterate nonsmooth guarantee, while the usual smooth SGD guarantee is recovered only as a special case when the reachable latent outputs remain in a compact region on which the soft-radial map is \(C^2\); see, \textit{e.g.}, \citet{ghadimi2013}, \citet{bottou2018}, \citet{clarke1990optimization}, and \citet{davis2020stochastic}.

\textit{Generic differentiability.} Let $\Sigma_p$ denote the set of points where the projection $p$ is non-differentiable. Because $p$ is locally Lipschitz (Lemma~\ref{lem:ray-regularity}), Rademacher's theorem dictates that $\Sigma_p$ has a Lebesgue measure of zero. Let $\Phi(\theta,z) \coloneqq g_\theta(z)$ represent the unconstrained network mapping. Assuming the parameters and inputs $(\theta,z)$ possess a joint density, and that the network mapping is non-degenerate (\textit{i.e.}, its Jacobian maintains full rank, $\operatorname{rank}D_{(\theta,z)}\Phi=n$ almost everywhere), the coarea formula guarantees that the network's output $\Phi(\theta,z)$ has a distribution absolutely continuous with respect to Lebesgue measure. Because the network's output distribution is absolutely continuous and the non-differentiable set $\Sigma_p$ has zero volume, it follows that $\mathbb{P}(g_\theta(z) \in \Sigma_p) = 0$. Under these assumptions, backpropagation through $p$ is well-defined almost surely. For standard fully connected networks with a free output bias, this non-degenerate condition is immediate because the derivative of $g_\theta(z)$ with respect to the output bias has full row rank. If the law of $\theta_t$ remains absolutely continuous along training and fresh samples $z_t$ are drawn from a density, the same argument applies at every finite iterate. This absolute-continuity preservation holds, for example, under full-dimensional additive continuous noise, or may be imposed as an assumption on the stochastic update map. Importantly, this almost-sure differentiability statement does not imply $L$-smoothness of the population objective; a locally Lipschitz function can be differentiable almost everywhere while still having gradient jumps.

\begin{assumption}[Compact stochastic approximation regime]
\label{ass:compact-sa}
The iterates generated by stochastic subgradient descent remain in a compact set
\(K\subset\mathbb R^{d_\theta}\). The objective $F(\theta) \coloneqq \mathbb E_z[\ell(p(g_\theta(z)))]$ is locally Lipschitz and tame on a neighborhood of \(K\), and is bounded below on \(K\). Moreover, the stochastic oracle satisfies the standard measurability,
bounded-noise, and subgradient-oracle conditions of tame nonsmooth stochastic
approximation.
\end{assumption}

\begingroup
\renewcommand{\theproposition}{3.6}
\begin{proposition}\label{prop:convergence}\textnormal{(Compact-iterate nonsmooth convergence).}
Suppose Assumption~\ref{ass:compact-sa} holds. Then stochastic subgradient descent with step sizes $\eta_t\to0$, $\sum_t\eta_t=\infty$, $\sum_t\eta_t^2<\infty$ satisfies the subsequential convergence guarantee of \citet[Cor.~5.9]{davis2020stochastic}: almost surely, every accumulation point of \(\{\theta_t\}\) is Clarke stationary. Equivalently, for every convergent subsequence \(\theta_{t_j}\to\bar\theta\), $0\in\partial^\circ F(\bar\theta) \textnormal{ a.s.}$, where \(\partial^\circ F\) denotes the Clarke subdifferential.
\end{proposition}
\endgroup

\begin{proof}
By Lemma~\ref{lem:ray-regularity}, the soft-radial map \(p\) is locally Lipschitz. Therefore, whenever \(g_\theta\) and \(\ell\) are locally Lipschitz on compact sets, the single-sample composite $\theta\mapsto \ell(p(g_\theta(z)))$ is locally Lipschitz on compact parameter sets. In the empirical finite-sum setting,
\[
    F(\theta)=\frac1m\sum_{i=1}^m\varphi_i(\theta),
    \qquad
    \varphi_i(\theta)\coloneqq \ell(p(g_\theta(z_i))),
\]
local Lipschitzness is preserved under finite summation, so \(F\) is locally Lipschitz on a neighborhood of \(K\). For general population expectations, local Lipschitzness can be verified using a uniform or integrable Lipschitz envelope, or assumed directly as in the proposition.

Tameness is also preserved under finite compositions and finite sums, provided the functions involved are definable in a common o-minimal structure. This covers standard semialgebraic settings, including ReLU networks, polyhedral feasible sets, and algebraic radial maps. Logistic sigmoid and \(\tanh\)-type activations can be handled in o-minimal structures containing the exponential function; otherwise, tameness should be verified or assumed directly.

On compact sets where the summands have a common Lipschitz constant, Clarke subgradients are uniformly bounded. Together with the compact-iterate, measurability, bounded-noise, and subgradient-oracle assumptions stated in the proposition, the tame nonsmooth stochastic approximation result of \citet[Cor.~5.9]{davis2020stochastic} applies. Hence every accumulation point of the stochastic subgradient iterates is Clarke stationary almost surely.
\end{proof}

\textit{Smooth special case.} The classical smooth SGD convergence rate can be recovered under a stronger "compact-patch" condition, which ensures the network's outputs remain within a smooth region of the projection layer. Let $V \supset K$ be an open convex neighborhood around the compact iterate set $K$, and let $\mathcal{Z}$ be compact with $\operatorname{supp}(\mathbb{P}_Z) \subseteq \mathcal{Z}$. We define the reachable latent set as $\mathcal{U}_V \coloneqq \{g_\theta(z) : \theta \in \overline{V}, z \in \mathcal{Z}\}$. Suppose $\mathcal{U}_V$ is compactly contained in an open set $U$ where the projection is twice continuously differentiable ($p \in C^2(U)$). Furthermore, assume the loss $\ell$ is $C^2$ on a neighborhood of $p(\mathcal{U}_V)$, and that the network $g_\theta$ has uniformly bounded first and second derivatives with respect to $\theta$ on $\overline{V} \times \mathcal{Z}$. Under these conditions, the composite objective $F$ has a bounded Hessian on $V$, making $F$ globally $L$-smooth on $V$ for some finite $L$. To see why, let $h(u) \coloneqq \ell(p(u))$. For a fixed input $z$, define the sample loss as $\psi_z(\theta) \coloneqq h(g_\theta(z))$. Applying the chain rule yields the gradient and Hessian:
\[
\nabla_\theta\psi_z(\theta)=D_\theta g_\theta(z)^\top\nabla h(g_\theta(z)),
\]
\[
    \nabla^2_{\theta\theta}\psi_z(\theta)
    = D_\theta g_\theta(z)^\top\nabla^2 h(g_\theta(z))
    D_\theta g_\theta(z)
    + \sum_i\partial_i h(g_\theta(z))\,\nabla^2_{\theta\theta} g_{\theta,i}(z).
\]
Because of our compactness and bounded-derivative assumptions, all terms on the right-hand side are uniformly bounded on $\overline{V} \times \mathcal{Z}$. This allows us to interchange differentiation and expectation, proving that $\nabla^2F$ is uniformly bounded on $V$. Consequently, if the stochastic-gradient oracle is unbiased and has bounded variance on $K$, the standard nonconvex SGD bound of \citet{ghadimi2013} applies. For a horizon-dependent constant step size $\eta_t \equiv \eta = \Theta(T^{-1/2})$ with $\eta \le 1/L$, we obtain
\[
    \min_{0\le t<T} \mathbb E\|\nabla F(\theta_t)\|^2=\mathcal{O}(T^{-1/2}).
\]

\textit{Geometric intuition.} This compact-patch condition is fundamentally geometric. For smooth convex bodies, such as balls and ellipsoids, the ray-boundary intersection varies smoothly wherever it is transverse. However, for polyhedral sets—such as boxes, simplices, and capped simplices—the active boundary face changes depending on the ray direction, making $p$ only piecewise smooth. In these cases, the smooth SGD guarantee only holds while the reachable set $\mathcal{U}_V$ stays a positive distance away from the boundary-switching and active-face switching surfaces. If the training iterates cross these surfaces, the objective becomes nonsmooth, and the Clarke-stationary guarantee detailed in Proposition~\ref{prop:convergence} becomes the appropriate theoretical framework.

\subsection{Proofs of Section \ref{sec:const-ua}}
\label{app:proofs-ua}

Finally, we prove that the constrained model inherits the universal approximation capabilities of the underlying unconstrained architecture.

\ConstUnvAppr*

\begin{proof} Without loss of generality, we assume $0 \in \operatorname{Int}(\mathcal{C})$ serves as the anchor point (see Coordinate Convention in Section~\ref{sec:soft-projection}). Throughout this proof, we denote the uniform norm of a function $\psi: \mathcal{Z} \to \mathbb{R}^n$ by $\|\psi\|_\infty \coloneqq \sup_{z \in \mathcal{Z}} \|\psi(z)\|$. We separate the interior and boundary cases.

\emph{Interior case.} Assume the target image satisfies $h(\mathcal{Z}) \subset \operatorname{Int}(\mathcal{C})$. Since $p: \mathbb{R}^n \to \operatorname{Int}(\mathcal{C})$ is a homeomorphism (Theorem~\ref{thm:homeomorphism}), the function $\phi \coloneqq p^{-1} \circ h: \mathcal{Z} \to \mathbb{R}^n$ is continuous. Since $\mathcal{Z}$ is compact, the image $\phi(\mathcal{Z})$ is compact. Fix any $\delta_0 > 0$ (\textit{e.g.}, $\delta_0 = 1$) and define the compact $\delta_0$-thickening
\[K \coloneqq \bigl\{ \xi \in \mathbb{R}^n \mid \operatorname{dist}\bigl(\xi, \phi(\mathcal{Z})\bigr) \le \delta_0 \bigr\},\]
where $\operatorname{dist}(\xi, A) \coloneqq \inf_{a \in A} \|\xi - a\|$. By the local Lipschitz continuity of $p$ on the compact set $K$, there exists a constant $L_K < \infty$ such that $\|p(a) - p(b)\| \le L_K \|a - b\|$ for all $a, b \in K$. By the universality of $\mathcal{G}$, there exists a function $g \in \mathcal{G}$ such that
\[\|g - \phi\|_\infty \;\le\; \delta,
\qquad
\delta \coloneqq \min\Bigl\{\delta_0,\ \frac{\varepsilon}{L_K}\Bigr\}.\]
This implies $g(\mathcal{Z}) \subset K$. Consequently, for all $z \in \mathcal{Z}$:
\begin{align*}
\|p(g(z)) - h(z)\|&= \|p(g(z)) - p(\phi(z))\| \\
&\le L_K \, \|g(z) - \phi(z)\| \\
&\le L_K \, \delta \;\le\; \varepsilon,
\end{align*}
which implies $\|p \circ g - h\|_\infty \le \varepsilon$, proving the claim for the interior case.

\emph{Boundary case.} For a general target $h: \mathcal{Z} \to \mathcal{C}$, we define a strictly interior approximation $h_\varepsilon$ by shrinking $h$ towards the anchor. Define the shrink operator $S_\tau(x) \coloneqq (1 - \tau)x$. Since $\mathcal{C}$ is convex and $0 \in \operatorname{Int}(\mathcal{C})$, we have $S_\tau(\mathcal{C}) \subset \operatorname{Int}(\mathcal{C})$ for any $\tau \in (0,1)$. Let $D \coloneqq \|h\|_\infty < \infty$. If $D = 0$, then $h \equiv 0$, and the interior case applies directly. Otherwise, we select a shrinkage factor $\tau$ sufficiently small to ensure $h_\varepsilon$ remains within $\varepsilon/2$ of $h$. Specifically, set:
\[
\tau \coloneqq \min\Bigl\{\tfrac{1}{2},\ \frac{\varepsilon}{2D}\Bigr\} \in (0, 1),
\qquad
h_\varepsilon \coloneqq S_\tau \circ h.
\]
The uniform distance between the original and shrunken target is bounded by
\[\|h_\varepsilon - h\|_\infty \;=\; \sup_{z \in \mathcal{Z}} \|(1 - \tau)h(z) - h(z)\| \;=\; \tau D \;\le\; \frac{\varepsilon}{2}.
\]
Since the new target satisfies $h_\varepsilon(\mathcal{Z}) \subset \operatorname{Int}(\mathcal{C})$, we can invoke the \emph{Interior case} result with precision $\varepsilon/2$. This guarantees the existence of a function $g \in \mathcal{G}$ such that
\[\|p \circ g - h_\varepsilon\|_\infty \;\le\; \frac{\varepsilon}{2}.\]
Finally, applying the triangle inequality yields the total error bound
\[\|p \circ g - h\|_\infty
\;\le\;
\underbrace{\|p \circ g - h_\varepsilon\|_\infty}_{\le \varepsilon/2} \;+\; \underbrace{\|h_\varepsilon - h\|_\infty}_{\le \varepsilon/2}
\;\le\; \varepsilon.\]
\end{proof}

\section{Extended Numerical Results}
\label{app:extended-numerical-results}

In this section, we first provide additional insights into the model parameters, specifically how the radial contraction function $r(\cdot)$ and the scale parameter $\lambda$ influence the effective optimization landscape. Next, we provide implementation details for the competitor algorithms, DC3 \citep{dontidc3} and HardNet \citep{min2024hardnet}, which were adjusted to accommodate the (capped) simplex constraints considered in our experiments. We then detail the architectures and hyperparameters utilized within our study. Finally, we provide the experimental details needed to reproduce the reported benchmarks.

\paragraph{Computational resources.}
All experiments reported in the paper were run on a shared compute cluster using containerized single-GPU jobs. Each job was allocated one NVIDIA A100-SXM4-80GB GPU and 16 CPU cores. Training times (s) and inference times (ms) are reported for each experiment where applicable.

\paragraph{Affine-hull parametrization.}
For simplex and capped-simplex constraints, the feasible set lies in the affine hyperplane defined by the budget equality. For \(\mathcal C\subseteq\{w\in\mathbb R^N:\mathds{1}^\top w=1\}\), where \(N\) denotes the number of assets or zones, a raw network output \(u=g_\theta(z)\in\mathbb R^N\) defines the centered direction
\[
d=(u-u_0)-\frac{1}{N}\mathds{1}\mathds{1}^{\top}(u-u_0).
\]
The ray search is then performed along \(u_0+td\). Since \(\mathds{1}^\top d=0\), the entire ray remains in the affine hyperplane \(\mathds{1}^\top w=1\), and the boundary computation only needs to enforce the inequality constraints. The portfolio and ride-sharing experiments use this parametrization; ride-sharing applies it in the normalized unit-budget capped simplex before rescaling by the time-varying supply \(S_t\).

\subsection{Sensitivity Analysis: Radial Contraction and Scaling}
\label{subsec:radial-contraction}

In Sec.~\ref{subsec:implementation-details} we introduced radial contraction functions $r(\cdot)$ satisfying the properties of Assumption~\ref{ass:radial}. For completeness, we restate them here. With fixed parameters $\varepsilon\in(0,1)$ and $\lambda>0$, we consider three families:
\begin{equation}
    r(\rho) = \varepsilon + (1-\varepsilon)\frac{\rho}{\rho+\lambda}
    \label{eq:rational}
\end{equation}
\begin{equation}
    r(\rho) = \varepsilon + (1-\varepsilon)\bigl(1 - e^{-\rho/\lambda}\bigr)
    \label{eq:exponential}
\end{equation}
\begin{equation}
    r(\rho) = \varepsilon + (1-\varepsilon)\tanh(\rho/\lambda)
    \label{eq:hyperbolic}
\end{equation}
The \emph{rational} radial contraction decays polynomially ($\mathcal{O}(\rho^{-1})$), whereas the \emph{exponential} and \emph{hyperbolic} forms exhibit exponential decay. Consequently, the rational form preserves larger gradient magnitudes for predictions falling far outside the feasible set ($\rho \gg 1$), which helps mitigate the vanishing gradient problem for extreme outliers.

Similar to the problem illustrated in Fig.~\ref{fig:projection}, we evaluate the role of $r(\cdot)$ in terms of warping the effective loss landscape (A), influencing the training loss (B), and determining the convergence to the target $x^*$. Figures~\ref{fig:soft-radial-proj-frac}, \ref{fig:soft-radial-proj-exp}, and \ref{fig:soft-radial-proj-tanh} visualize the effective loss landscape for the \emph{rational} \eqref{eq:rational}, \emph{exponential} \eqref{eq:exponential}, and \emph{hyperbolic} \eqref{eq:hyperbolic} functions, respectively. Furthermore, we illustrate in Fig.~\ref{fig:soft-radial-proj-epsilon} how the minimum scaling threshold $\varepsilon$ modifies the landscape geometry surrounding the chosen anchor $u_0$.

\begin{figure*}[!htbp]
    \centering
    \begin{subfigure}[b]{0.325\textwidth}
        \centering
        \includegraphics[width=\textwidth]{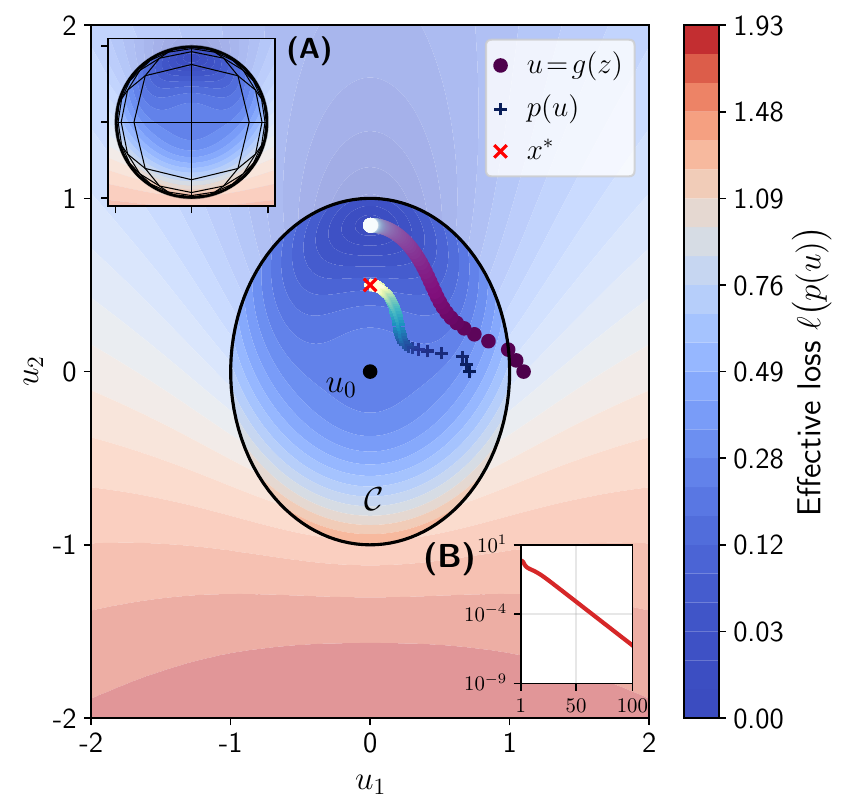}
        \caption{$\lambda=0.5$.}
        \label{fig:soft-radial-proj-layer-frac-lambda-0p5}
    \end{subfigure}
    \begin{subfigure}[b]{0.325\textwidth}
        \centering
        \includegraphics[width=\textwidth]{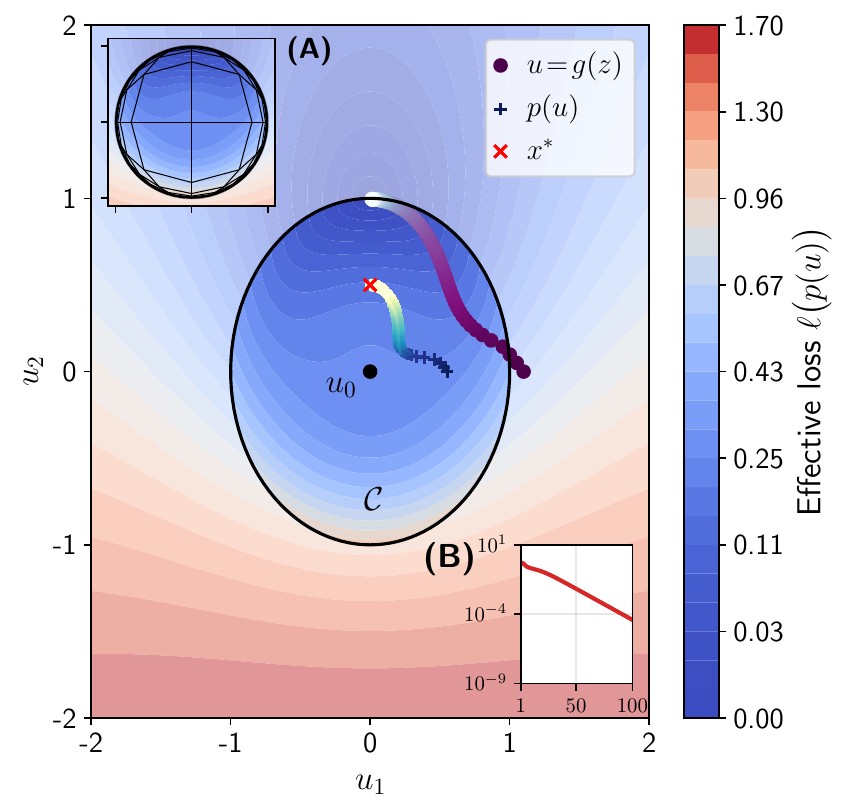}
        \caption{$\lambda=1.0$.}
        \label{fig:soft-radial-proj-layer-frac-lambda-1p0}
    \end{subfigure}
    \begin{subfigure}[b]{0.325\textwidth}
        \centering
        \includegraphics[width=\textwidth]{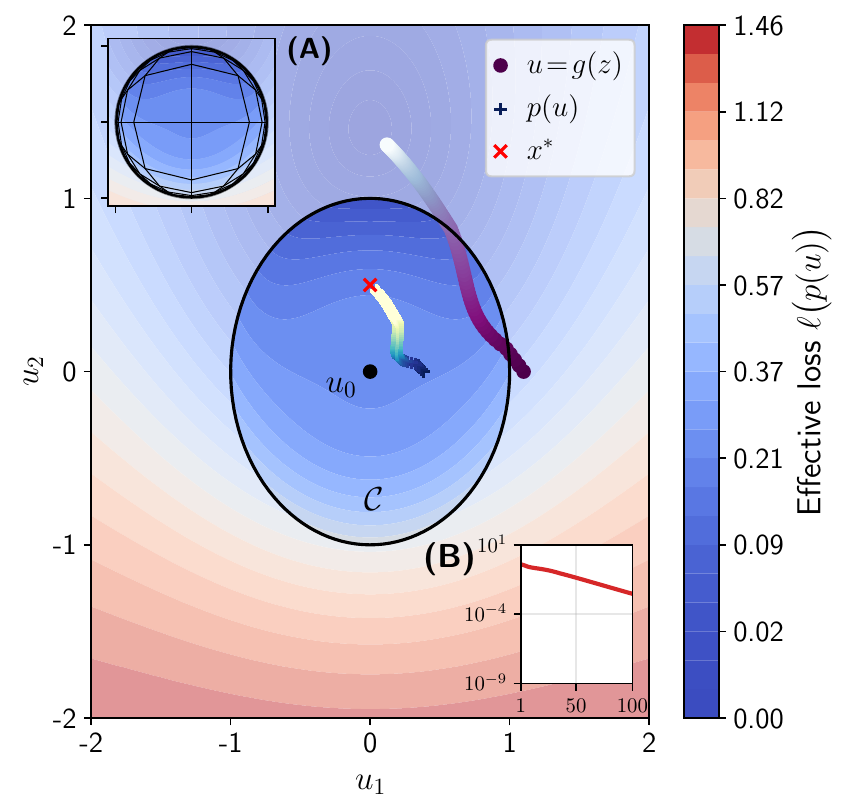}
        \caption{$\lambda=2.0$}
        \label{fig:soft-radial-proj-layer-frac-lambda-2p0}
    \end{subfigure}
    \caption{Soft-radial projection $p(u)$ with \emph{rational} radial contraction \eqref{eq:rational}.}
    \label{fig:soft-radial-proj-frac}
\end{figure*}

\begin{figure*}[!htbp]
    \centering
    \begin{subfigure}[b]{0.325\textwidth}
        \centering
        \includegraphics[width=\textwidth]{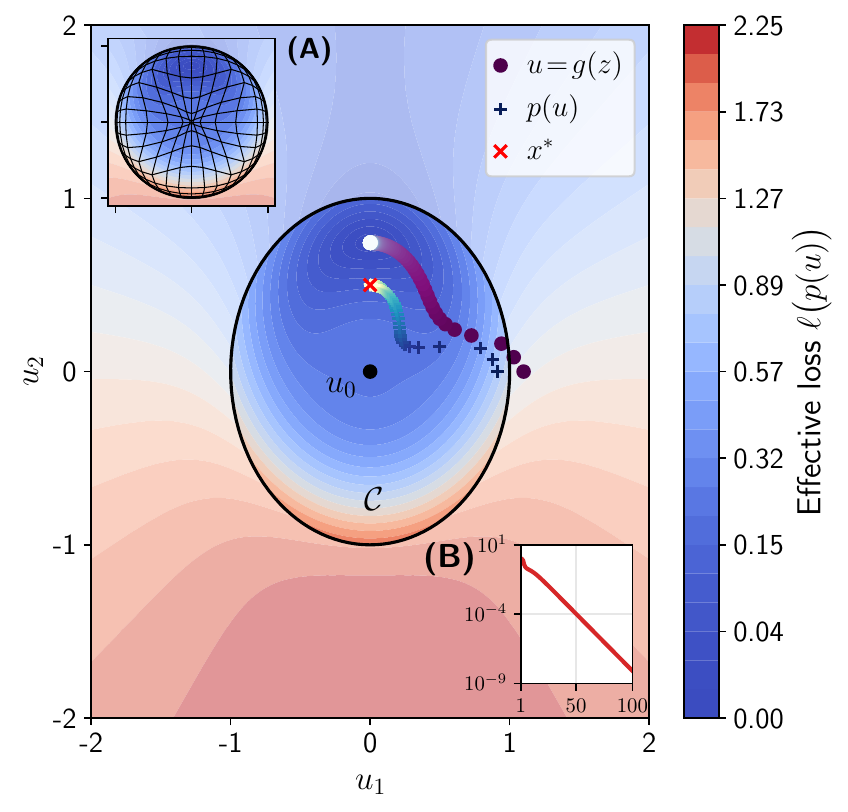}
        \caption{$\lambda=0.5$.}
        \label{fig:soft-radial-proj-layer-exp-lambda-0p5}
    \end{subfigure}
    \begin{subfigure}[b]{0.325\textwidth}
        \centering
        \includegraphics[width=\textwidth]{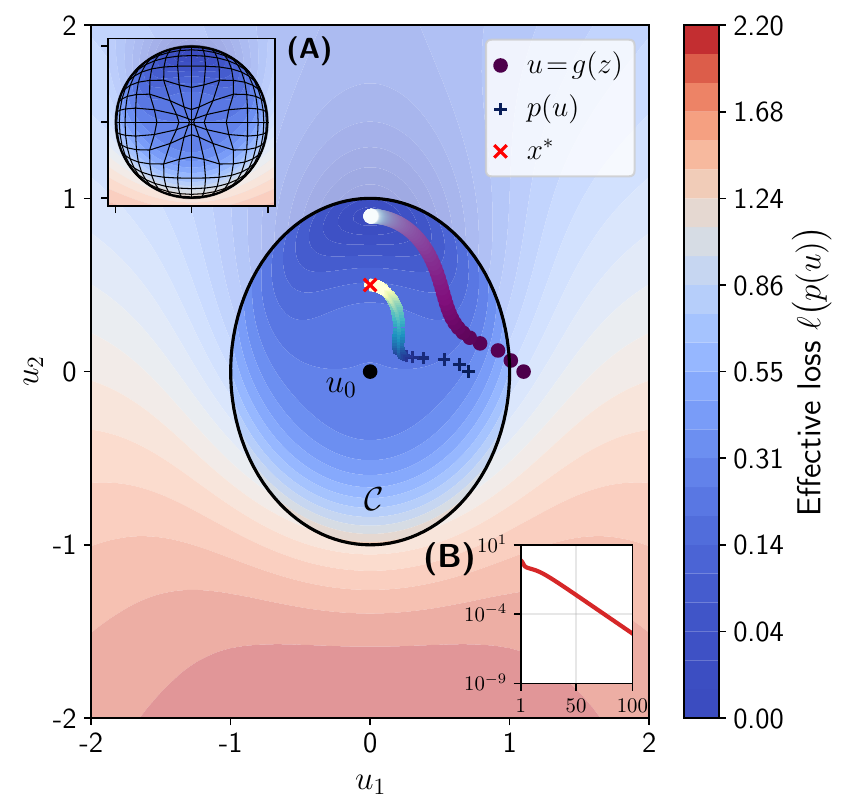}
        \caption{$\lambda=1.0$.}
        \label{fig:soft-radial-proj-layer-exp-lambda-1p0}
    \end{subfigure}
    \begin{subfigure}[b]{0.325\textwidth}
        \centering
        \includegraphics[width=\textwidth]{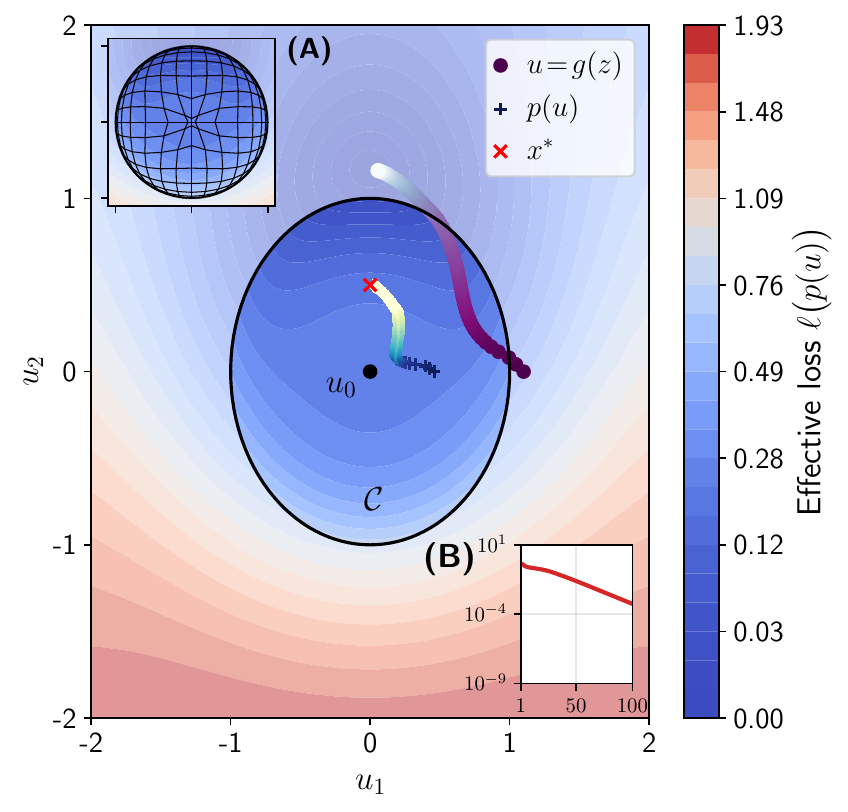}
        \caption{$\lambda=2.0$.}
        \label{fig:soft-radial-proj-layer-exp-lambda-2p0}
    \end{subfigure}
    \caption{Soft-radial projection $p(u)$ with \emph{exponential} radial contraction \eqref{eq:exponential}.}
    \label{fig:soft-radial-proj-exp}
\end{figure*}

\begin{figure*}[!htbp]
    \centering
    \begin{subfigure}[b]{0.325\textwidth}
        \centering
        \includegraphics[width=\textwidth]{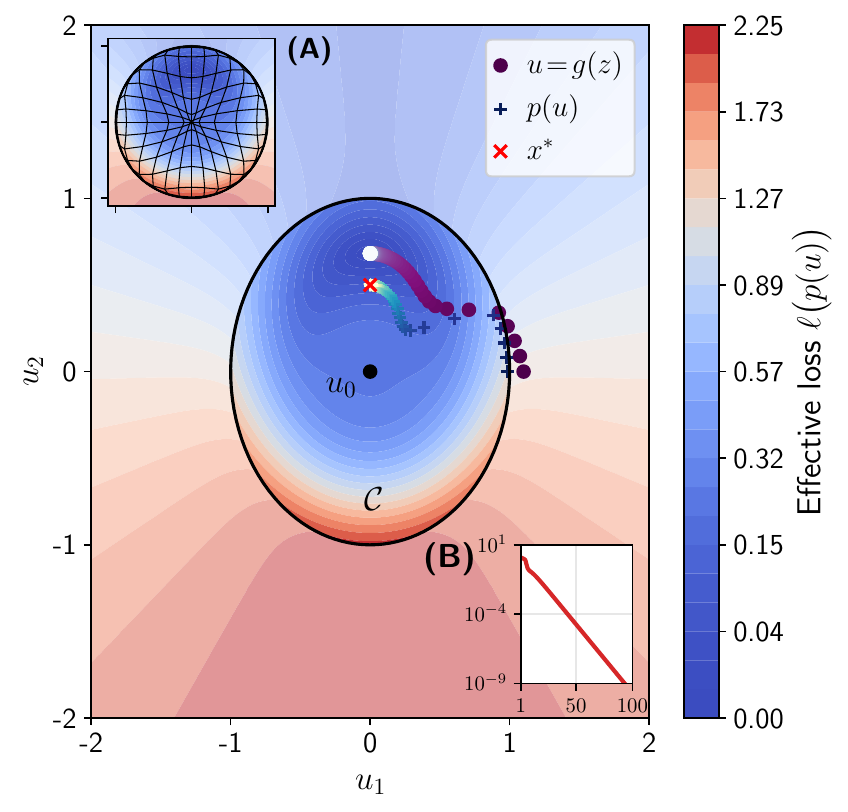}
        \caption{$\lambda=0.5$.}
        \label{fig:soft-radial-proj-layer-tanh-lambda-0p5}
    \end{subfigure}
    \begin{subfigure}[b]{0.325\textwidth}
        \centering
        \includegraphics[width=\textwidth]{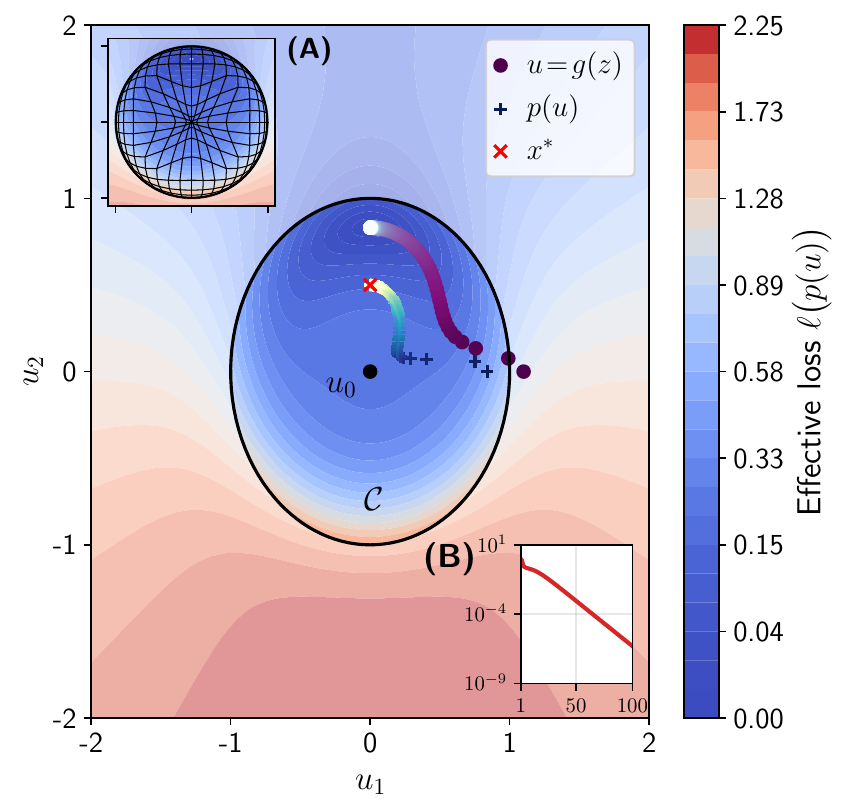}
        \caption{$\lambda=1.0$.}
        \label{fig:soft-radial-proj-layer-tanh-lambda-1p0}
    \end{subfigure}
    \begin{subfigure}[b]{0.325\textwidth}
        \centering
        \includegraphics[width=\textwidth]{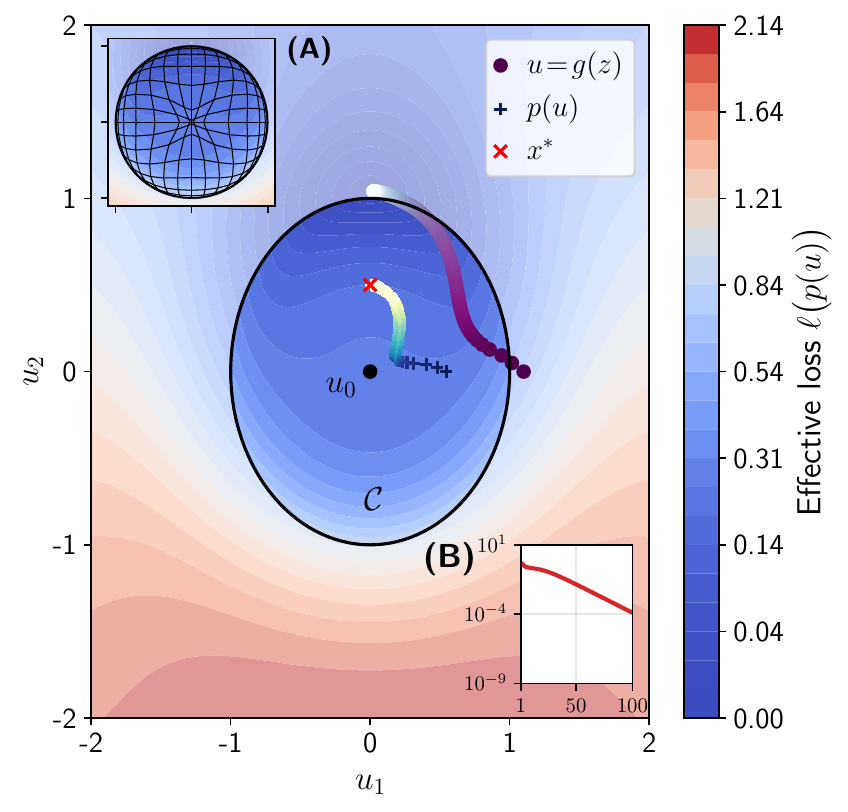}
        \caption{$\lambda=2.0$.}
        \label{fig:soft-radial-proj-layer-tanh-lambda-2p0}
    \end{subfigure}
    \caption{Soft-radial projection $p(u)$ with \emph{hyperbolic} radial contraction \eqref{eq:hyperbolic}.}
    \label{fig:soft-radial-proj-tanh}
\end{figure*}

\begin{figure*}[!htbp]
    \centering
    \begin{subfigure}[b]{0.325\textwidth}
        \centering
        \includegraphics[width=\textwidth]{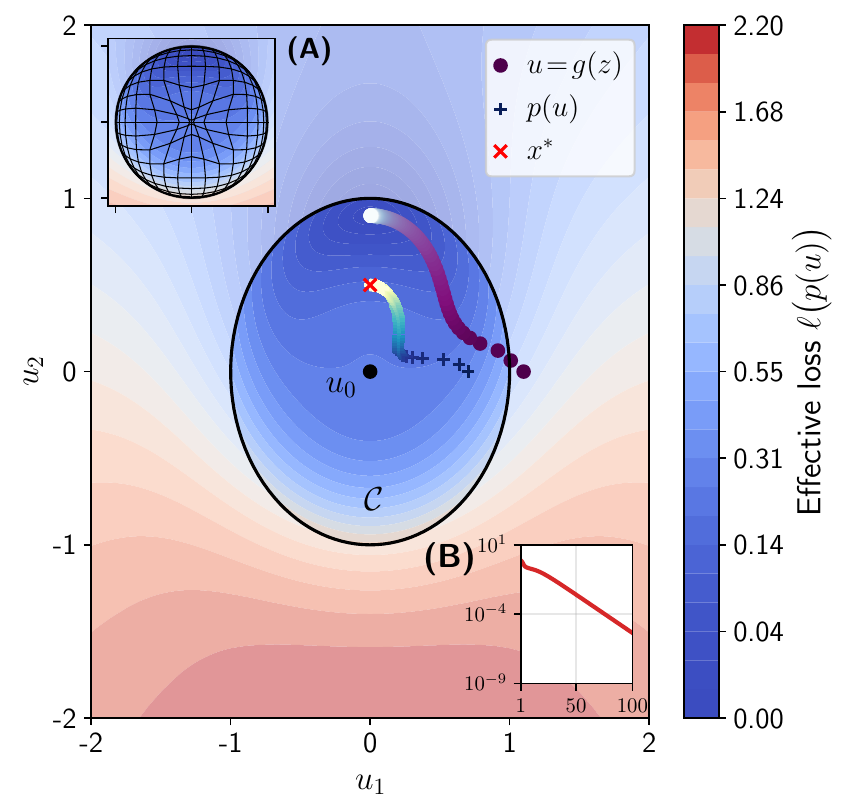}
        \caption{$\varepsilon=0.001$.}
        \label{fig:soft-radial-proj-epsilon-0p001}
    \end{subfigure}
    \begin{subfigure}[b]{0.325\textwidth}
        \centering
        \includegraphics[width=\textwidth]{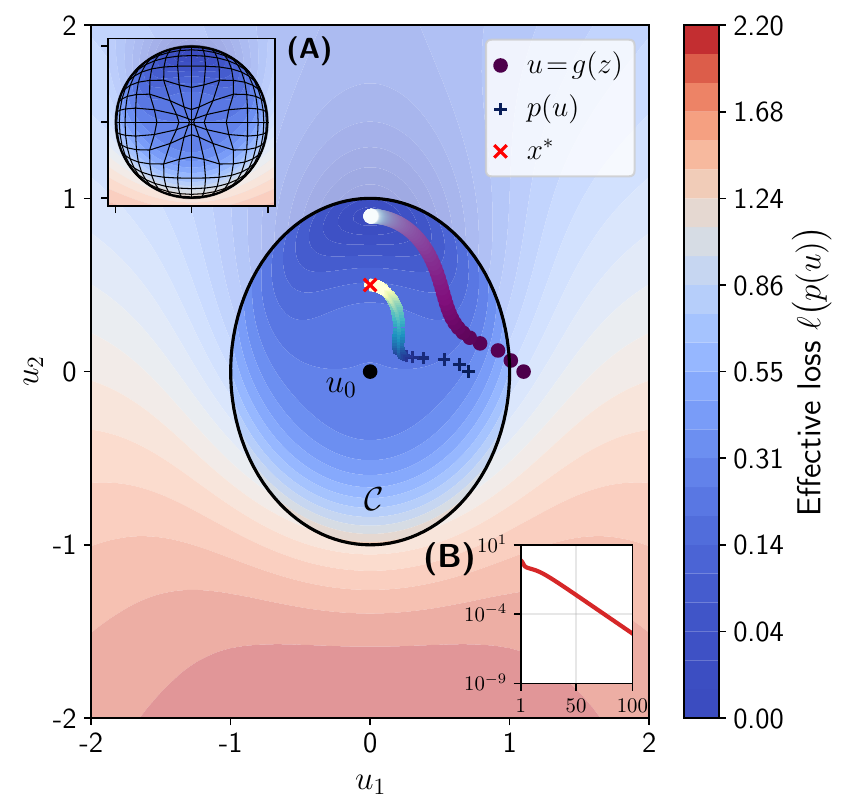}
        \caption{$\varepsilon=0.01$.}
        \label{fig:soft-radial-proj-epsilon-0p01}
    \end{subfigure}
    \begin{subfigure}[b]{0.325\textwidth}
        \centering
        \includegraphics[width=\textwidth]{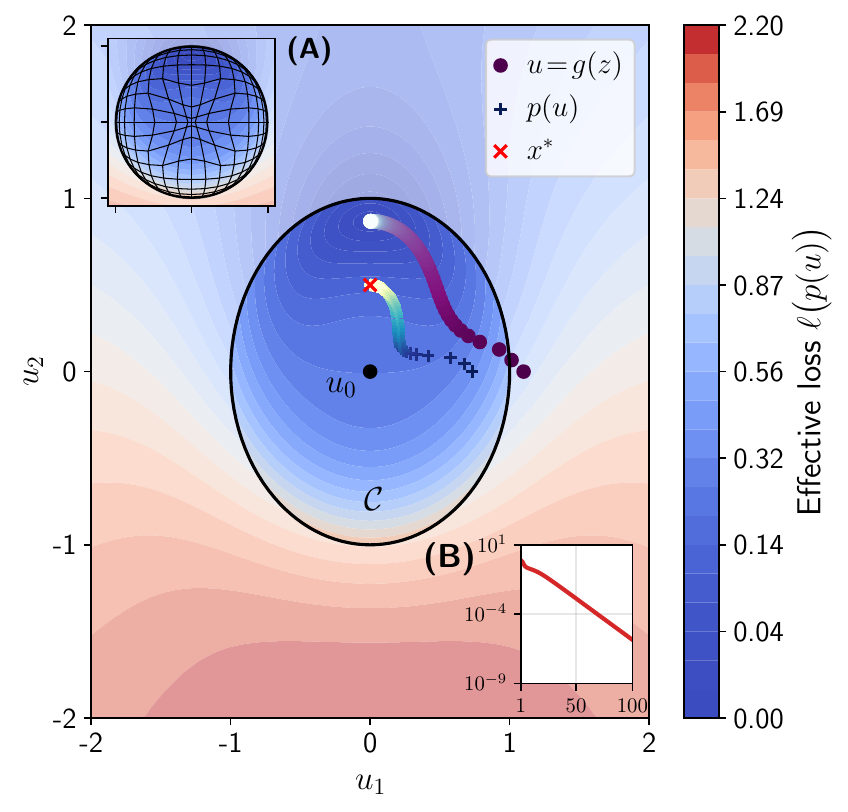}
        \caption{$\varepsilon=0.1$.}
        \label{fig:soft-radial-proj-epsilon-0p1}
    \end{subfigure}
    \caption{Soft-radial projection $p(u)$ for varying minimum scaling threshold $\varepsilon$.}
    \label{fig:soft-radial-proj-epsilon}
\end{figure*}

\newpage
\subsection{Adaptation of Competitor Methods to Capped Simplex Constraints}
\label{app:competitor-methods}

We explain the adaptation of the competitor methods with respect to the capped simplex introduced in the portfolio setting and formally defined as
\[
\mathcal C
\;=\;
\Bigl\{w\in\mathbb R^N \ \big|\ 0\leq  w,\ \mathds{1}^\top w = 1,\  w \leq c\Bigr\},
\]
where $w$ is the portfolio weight and $c$ is a cap on the weight.

\textbf{HardNet.} The HardNet framework \cite{min2024hardnet} enforces feasibility by appending a differentiable projection layer to an unconstrained predictor $g_\theta:\mathcal Z\to\mathbb R^N$. Given an input $z\in\mathcal Z$ and raw output $u \coloneqq g_\theta(z)$, the method considers input-dependent affine constraints of the form
\[
\underline{b}(z) \leq A(z) w \leq \bar{b}(z).
\]
\citet{min2024hardnet} derive a closed-form projection for this case, denoted \textit{HardNet-Aff}:
\begin{equation}
\label{eq:hardnet-aff}
P_{\text{HN}}(u;z)
\;=\;
u
\;+\;
A^\dagger_\mathrm{R}(z)\Big(
\operatorname{ReLU}\big(\underline{b}(z)-A(z)u\big)
-
\operatorname{ReLU}\big(A(z)u- \bar{b}(z)\big)
\Big),
\end{equation}
where $A^\dagger_\mathrm{R}(z) \coloneqq A(z)^\top\!\big(A(z)A(z)^\top\big)^{-1}$ is the Moore--Penrose \emph{right} pseudoinverse. This closed form assumes that $A(z)$ has full row rank (Assumption~5 in \cite{min2024hardnet}), describing an underdetermined system.
In our portfolio setting, the feasible set is a capped simplex; stacking the equality constraint with the $2N$ bound constraints yields an \emph{overdetermined} affine system. Consequently, the matrix $A A^\top$ is singular and the right-pseudoinverse expression in~\eqref{eq:hardnet-aff} is not directly applicable. We therefore use a least-squares relaxation of the HardNet correction. This involves defining a (signed) violation vector 
\begin{equation}
\label{eq:hardnet-viol}
v(u;z)
\;\coloneqq\;
\operatorname{ReLU}\big( \underline{b}(z)-A u\big)
-
\operatorname{ReLU}\big(A u-\bar{b}(z)\big).
\end{equation}
Since the resulting static constraint matrix $A$ has full column rank in our construction, we compute a least-squares correction direction via the \emph{left} pseudoinverse
\begin{equation}
\label{eq:hardnet-lsq}
\Delta(u;z)
\;\coloneqq\;
\arg\min_{\Delta\in\mathbb R^N}\ \|A\Delta - v(u;z)\|^2
\;=\;
A^\dagger_{\mathrm{L}}\,v(u;z),
\qquad
A^\dagger_{\mathrm{L}}\coloneqq (A^\top A)^{-1}A^\top.
\end{equation}
We then set $\tilde{w}\coloneqq u+\Delta(u;z)$. A conceptually broader but computationally slower alternative proposed by the authors for general convex constraints is to embed the projection as a differentiable convex optimization layer \cite{agrawal2019differentiable}, referred to as \textit{HardNet-Cvx}. To remove any remaining numerical residuals during evaluation, we additionally apply an exact orthogonal projection onto the capped simplex, $w = P(\tilde{w})$, to guarantee $w \in \mathcal C$.

\paragraph{Symmetric DC3 for the capped simplex.}
We adapt the Deep Constraint Completion and Correction (DC3) framework \citep{dontidc3} to $\mathcal C$. Standard DC3 selects $N-1$ independent variables and defines the remaining coordinate as dependent via the equality constraint. Applied directly to the raw network output $u=g_\theta(z)$, this introduces an asymmetry: the dependent coordinate absorbs the entire sum deviation, often producing unstable magnitudes and gradients. To mitigate this, we employ a \emph{symmetric} initialization that first distributes the sum deviation evenly by orthogonally projecting $u$ onto the simplex hyperplane
\(
\mathcal H \coloneqq \{ w\in\mathbb R^N:\mathbf 1^\top  w=1\}.
\)
Concretely, we set
\begin{equation}
\label{eq:dc3-init}
w^{(0)}
\;\coloneqq\;
P_{\mathcal H}(u)
\;=\;
u-\frac{1}{N}\big(\mathds{1}^\top u - 1\big)\mathds{1}
\;\in\;
\mathcal H.
\end{equation}
We then take the first $N-1$ coordinates as free variables
\(
\xi\coloneqq (w^{(0)}_1,\ldots,w^{(0)}_{N-1})^\top\in\mathbb R^{N-1}
\)
and enforce the equality constraint through a completion map
\begin{equation}
\label{eq:dc3-complete}
w(\xi)
\;\coloneqq\;
\begin{bmatrix}
\xi \\
1-\mathds{1}^\top \xi
\end{bmatrix},
\end{equation}
which ensures $\mathds{1}^\top w(\xi) = 1$ for all $\xi$.
To correct violations of the bound constraints $ 0\leq w \leq c$, we minimize a squared hinge-residual energy
\begin{equation}
\label{eq:dc3-energy}
\mathcal V(\xi)
\;\coloneqq\;
\sum_{i=1}^N
\Big(
\operatorname{ReLU}\big(-w_i(\xi)\big)^2
+
\operatorname{ReLU}\big(w_i(\xi)-c_i\big)^2
\Big),
\end{equation}
via an unrolled gradient descent loop on the free variables
\begin{equation}
\label{eq:dc3-update}
\xi^{(t+1)}
\;=\;
\xi^{(t)} - \eta\,\nabla_{\xi}\mathcal V\big(\xi^{(t)}\big),
\qquad t=0,\ldots,L-1.
\end{equation}
The final output is given by $w^{(L)}= w(\xi^{(L)})$. Gradients are backpropagated through the completion map~\eqref{eq:dc3-complete}, properly accounting for the dependence of the dependent coordinate on the free variables. Similar to the HardNet implementation, we apply an orthogonal projection during evaluation to ensure strict feasibility.

For the standard simplex, orthogonal projection uses a sorting-based threshold algorithm; for the capped simplex, it uses bisection on the dual threshold. For network-constrained polytopes, we use Dykstra's algorithm with feasibility tolerance $10^{-6}$.

\subsection{Controlled Geometry Benchmarks}
\label{subsec:controlled-polytope-benchmarks}

The application benchmarks combine several sources of difficulty: data noise, architecture choice, loss geometry, and constraint enforcement. We therefore include two controlled benchmarks whose purpose is to isolate the geometry of the projection layer itself. Both are amortized convex quadratic programs over the same type of affine polytope,
\(\mathcal C=\{x\in\mathbb R^n:A_{\rm eq}x=b_{\rm eq},\,A_{\rm ineq}x\le b_{\rm ineq}\}\).
The feasible set is known exactly, and each sample has a known constrained optimizer. We can therefore vary one geometric property at a time while keeping the data-generating process fixed, and evaluate methods by optimality gap, training time, and inference latency.

\textbf{Boundary-margin benchmark.}
The first benchmark varies the distance between the downstream optimizer and the boundary. For each sample, we draw a direction from an interior reference point and place the quadratic minimizer at a prescribed relative margin \(\eta\) from \(\partial\mathcal C\). Smaller values of \(\eta\) generate instances whose optimum lies closer to a boundary face. This is the regime in which boundary-valued enforcement maps, such as hard radial maps or orthogonal projections, can lose informative gradients by mapping many infeasible predictions to the same face. The projection anchor \(u_0\) is fixed throughout, so changes in the results reflect sensitivity to the location of the optimum rather than changes in anchor quality.

\begin{table}[h]
\centering
\small
\caption{Boundary-margin sensitivity results aggregated over five seeds (mean $\pm$ std).}
\begin{tabular}{clccc}
\toprule
$\eta$ & Method & Rel. gap$\downarrow$ & Training [s]$\downarrow$ & Inference [ms]$\downarrow$\\
\midrule
0.1 & O-Proj & $0.07\,(0.01)$ & $689.43\,(50.03)$ & $135.40\,(52.87)$\\
0.1 & DC3 & $0.07\,(0.01)$ & $33.35\,(0.59)$ & $25.74\,(2.72)$\\
0.1 & HardNet & $0.07\,(0.01)$ & $3.15\,(0.10)$ & $1.17\,(0.15)$\\
0.1 & RAYEN & $0.03\,(0.00)$ & $2.60\,(0.47)$ & $1.10\,(0.23)$\\
0.1 & SRP & $0.02\,(0.00)$ & $3.43\,(0.16)$ & $1.27\,(0.32)$\\
\midrule
0.01 & O-Proj & $0.07\,(0.01)$ & $800.86\,(40.53)$ & $180.47\,(73.27)$ \\
0.01 & DC3 & $0.07\,(0.01)$ & $33.48\,(1.60)$ & $27.73\,(2.34)$\\
0.01 & HardNet & $0.07\,(0.01)$ & $2.75\,(0.07)$ & $0.99\,(0.05)$\\
0.01 & RAYEN & $0.03\,(0.00)$ & $2.56\,(0.38)$ & $0.93\,(0.16)$\\
0.01 & SRP & $0.03\,(0.00)$ & $2.67\,(0.17)$ & $0.90\,(0.09)$\\
\midrule
0.0 & O-Proj & $0.07\,(0.01)$ & $803.97\,(19.46)$ & $224.84\,(104.72)$ \\
0.0 & DC3 & $0.07\,(0.01)$ & $33.99\,(1.57)$ & $34.54\,(8.36)$ \\
0.0 & HardNet & $0.07\,(0.01)$ & $2.67\,(0.13)$ & $0.96\,(0.11)$  \\
0.0 & RAYEN & $0.03\,(0.00)$ & $2.60\,(0.42)$ & $0.94\,(0.21)$ \\
0.0 & SRP & $0.03\,(0.00)$ & $2.10\,(0.15)$ & $0.69\,(0.02)$ \\
\bottomrule
\end{tabular}
\label{tab:boundary-results}
\end{table}

\textbf{Anchor-conditioning benchmark.}
The second benchmark holds the feasible set, data distribution, and constrained optima fixed, and varies only the interior anchor \(u_0\) used by radial methods. We move the anchor toward \(\partial\mathcal C\) through an anchor-slack parameter \(\delta\), with smaller \(\delta\) indicating a lower-slack anchor. Since the underlying quadratic programs do not change, variation across this table should be interpreted as conditioning induced by the projection geometry, not as a change in statistical or optimization difficulty.

\begin{table}[h]
\centering
\small
\caption{Anchor-sensitivity results aggregated over five seeds (mean $\pm$ std).}
\begin{tabular}{clccc}
\toprule
$\delta$ & Method & Rel. gap$\downarrow$ & Training [s]$\downarrow$ & Inference [ms]$\downarrow$\\
\midrule
0.1 & O-Proj & $0.07\,(0.01)$ & $561.24\,(24.90)$ & $168.92\,(64.50)$\\
0.1 & DC3 & $0.07\,(0.01)$ & $46.58\,(9.42)$ & $34.53\,(20.08)$ \\
0.1 & HardNet & $0.07\,(0.01)$ & $11.36\,(3.12)$ & $6.58\,(4.51)$ \\
0.1 & RAYEN & $0.04\,(0.00)$ & $7.95\,(1.11)$ & $1.93\,(1.15)$ \\
0.1 & SRP & $0.04\,(0.00)$ & $8.06\,(2.62)$ & $2.38\,(1.31)$\\
\midrule
0.01 & O-Proj & $0.07\,(0.01)$ & $571.92\,(17.42)$ & $205.95\,(74.30)$ \\
0.01 & DC3 & $0.07\,(0.01)$ & $43.76\,(6.59)$ & $50.04\,(24.16)$ \\
0.01 & HardNet & $0.07\,(0.01)$ & $2.52\,(0.07)$ & $0.77\,(0.30)$ \\
0.01 & RAYEN & $0.05\,(0.01)$ & $3.86\,(1.18)$ & $0.95\,(0.52)$ \\
0.01 & SRP & $0.04\,(0.00)$ & $2.89\,(0.45)$ & $0.78\,(0.15)$\\
\bottomrule
\end{tabular}
\label{tab:anchor-results}
\end{table}

\paragraph{Boundary optima and matched hard-radial control.}
To examine the effect of finite latent norm budgets when optimal solutions lie on the feasible boundary, we consider 1,024 quadratic optimization instances over a fixed polytope in $\mathbb R^{12}$ with intrinsic dimension eight, four equality constraints, and 32 inequalities. Each instance has an optimum on a designated facet, with boundary-pressure parameter $\beta=1$. SRP uses the rational contraction with $\lambda=1$ and $\varepsilon=0.05$. As a control, we introduce a matched hard-radial projection (Matched HR) that applies the same latent norm budget $M$ and contracts the hard-radial output toward the anchor by $r(M^2)$. This matches the limiting interior margin of SRP while retaining the many-to-one behavior of hard radial projection. Both methods use a two-layer MLP of width 64, trained for 50 epochs over five seeds. Performance is measured by the test-set mean objective gap, normalized by the improvement from the anchor to the constrained optimum.

\begin{table}[h]
\centering
\small
\caption{Relative regret for boundary-optimum instances under different latent norm budgets $M$, aggregated over five seeds (mean $\pm$ std). Lower is better.}
\label{tab:matched-hard-radial}
\begin{tabular}{lcc}
\toprule
$M$ & Matched HR & SRP \\
\midrule
1 & $0.4454\,(0.0035)$ & $0.4446\,(0.0034)$ \\
2 & $0.1979\,(0.0048)$ & $0.1962\,(0.0038)$ \\
4 & $0.0937\,(0.0057)$ & $0.0962\,(0.0044)$ \\
8 & $0.0632\,(0.0060)$ & $0.0711\,(0.0047)$ \\
16 & $0.0554\,(0.0060)$ & $0.0654\,(0.0046)$ \\
\bottomrule
\end{tabular}
\end{table}

Table~\ref{tab:matched-hard-radial} shows that relative regret decreases with the latent norm budget for both methods. Matched HR achieves lower mean regret for larger budgets, illustrating that injectivity alone does not guarantee superior optimization performance when the optimum lies on the boundary. Without an explicit latent norm budget, SRP achieves a relative regret of $0.0636\,(0.0046)$, compared with $0.0529\,(0.0061)$ for RAYEN.

\paragraph{End-to-end training and fine-tuning.}
To assess the applicability of SRP to pretrained neural networks, we compare direct end-to-end training with unconstrained pretraining followed by fine-tuning on 1,024 quadratic optimization instances over a fixed polytope in $\mathbb R^{12}$ with intrinsic dimension eight. Direct end-to-end training optimizes the constrained model for 50 epochs. Alternatively, the base network is pretrained for 20 epochs to predict the reduced coordinates of the constrained optimizers using mean squared error, followed by 30 epochs of end-to-end fine-tuning. A third setting evaluates the pretrained network immediately after attaching the constraint layer, without fine-tuning (zero-shot). Results are aggregated over five seeds. Performance is measured by the test-set mean objective gap relative to the absolute optimal objective value, using a different normalization from the anchor-relative regret reported above.

\begin{table}[h]
\centering
\small
\caption{Relative objective gap under different training strategies, aggregated over five seeds (mean $\pm$ std).}
\label{tab:pretraining-comparison}
\begin{tabular}{llc}
\toprule
Method & Training protocol & Relative gap$\downarrow$ \\
\midrule
\multirow{3}{*}{SRP}
& End-to-end (50 epochs) & $0.0264\,(0.0022)$ \\
& Pretrained, zero-shot & $0.2373\,(0.0225)$ \\
& Pretrained, fine-tuned (20+30 epochs) & $0.0292\,(0.0032)$ \\
\midrule
\multirow{3}{*}{RAYEN}
& End-to-end (50 epochs) & $0.0312\,(0.0030)$ \\
& Pretrained, zero-shot & $0.1762\,(0.0301)$ \\
& Pretrained, fine-tuned (20+30 epochs) & $0.0361\,(0.0049)$ \\
\bottomrule
\end{tabular}
\end{table}

Table~\ref{tab:pretraining-comparison} shows that end-to-end training and fine-tuning achieve comparable performance, with substantially larger objective gaps under zero-shot evaluation. SRP achieves lower mean gaps than RAYEN under both training strategies, whereas RAYEN performs better when the projection layer is applied without fine-tuning. These results demonstrate that SRP supports both direct end-to-end training and the adaptation of pretrained predictors.

\subsection{Portfolio and Ride-Sharing Hyperparameters}

Table \ref{tab:hyperparams} summarizes the hyperparameters used in the main experiments, distinguishing between the portfolio optimization and ride-sharing dispatch task. For both domains, we perform a grid search to select the optimal configuration based on validation performance. As noted earlier, for HardNet and DC3, we apply a final orthogonal projection to the candidate solution during evaluation only. This ensures feasibility of the reported decisions; the projection is not applied during training. Reported inference times include the evaluation-time correction.

\begin{table}[h]
    \centering
    \small
    \caption{Hyperparameter settings for portfolio optimization and ride-sharing dispatch experiments.}
    \label{tab:hyperparams}
    \begin{tabular}{l c c}
        \toprule
        \textbf{Parameter} & \textbf{Portfolio Optimization} & \textbf{Ride-Sharing Dispatch} \\
        \midrule
        \multicolumn{3}{l}{\textit{Architecture \& Training}} \\
        Base Architecture & LSTM & MLP \\
        Hidden Units & \multicolumn{2}{c}{$\{32, 64\}$} \\
        Dropout Rate & \multicolumn{2}{c}{$\{0.1, 0.2\}$} \\
        Optimizer & \multicolumn{2}{c}{Adam} \\
        Learning Rate & \multicolumn{2}{c}{$\{1 \times 10^{-4}, 5 \times 10^{-4}\}$} \\
        Lookback Horizon ($H$) & $\{10, 30\}$ & $24$ \\
        Batch Size & $64$ & $128$ \\
        Training Epochs & $50$ & $100$ \\
        \midrule
        \multicolumn{3}{l}{\textit{Projection Layer Hyperparameters}} \\
        Softmax Temperature ($\tau$) & \multicolumn{2}{c}{$\{0.1, 0.5, 1.0, 2.0\}$} \\
        Soft-radial $\lambda$ & \multicolumn{2}{c}{$\{0.5, 1.0, 2.0, 5.0, 10.0\}$} \\
        Soft-radial $\varepsilon$ & \multicolumn{2}{c}{0.05}\\
        Soft-radial $r(\cdot)$ & \multicolumn{2}{c}{\{Rational, Exponential, Hyperbolic\}} \\
        HardNet LS-correction iterations & \multicolumn{2}{c}{$\{1, 3, 5\}$} \\
        DC3 Steps & \multicolumn{2}{c}{$\{1, 3, 5\}$} \\
        DC3 Learning Rate & \multicolumn{2}{c}{$\{10^{-2}, 10^{-1}\}$} \\
        DC3 Momentum & \multicolumn{2}{c}{$\{0.5, 0.9\}$} \\
        \bottomrule
    \end{tabular}
\end{table}

\textbf{Results analysis.} Table~\ref{tab:boundary-results} shows that moving the constrained optimizer toward the boundary does not degrade the radial methods: RAYEN and SRP remain stable across all margins and achieve lower relative gaps than the projection and repair baselines. Table~\ref{tab:anchor-results} isolates a different effect. When the radial anchor is moved closer to the boundary, we observe a mild degradation, consistent with poorer conditioning of the ray-boundary map for low-slack anchors.

\subsection{Extended Analysis: Portfolio Optimization}
\label{subsec:extended-analysis-portfolio-optimization}

\subsubsection{Discussion: Deep Portfolio Optimization vs. Classical Approaches}
\label{subsubsec:discussion}

The formulation presented in Section~\ref{subsec:portfolio-optimization} represents a \textit{deep portfolio optimization} problem \cite{zhang2020deep}, which fundamentally differs from the classical mean-variance optimization framework \cite{markowitz1952}.

\textbf{Predict-then-optimize vs.~end-to-end learning.}
Classical mean-variance optimization typically operates in a ``predict-then-optimize'' framework: one first estimates the expected return vector $\mu\in\mathbb{R}^N$ and the covariance matrix $\Sigma\in\mathbb{S}^N_+$, and then solves a quadratic program to maximize $w^\top \mu - \zeta w^\top \Sigma w$. This two-step process suffers from two major limitations. First, it separates the prediction error from the optimization cost; a small error in estimating $\Sigma$ can lead to drastically suboptimal portfolios (a phenomenon formally known as the ``error maximization'' property of Markowitz \cite{michaud1989}). Second, while standard sample estimators neglect temporal dependencies (implicitly assuming i.i.d.~returns), more advanced dynamic estimators (\textit{e.g.}, generalized autoregressive conditional heteroskedasticity (GARCH) model) often impose rigid parametric structures that fail to capture the complex, non-linear interactions between assets that recurrent neural network architectures in the form of $g_\theta$ are designed to exploit.

\textbf{Parameter sensitivity and convexity.}
A significant hurdle in classical frameworks is the sensitivity to the risk-aversion parameter $\zeta$, which requires frequent recalibration. Our formulation bypasses this by directly maximizing the Sharpe ratio---a scale-invariant metric that eliminates the need for a subjective risk-return trade-off coefficient. While the static Sharpe ratio maximization is a quasiconcave problem that can be reformulated as a Second-Order Cone Program (SOCP) for global convergence, parameterizing the weights via a neural network $g_\theta$ introduces non-convexity. We accept this trade-off to leverage the representational power of deep learning; the model can ingest complex state histories to generate non-linear weight mappings, far surpassing the constraints of the linear-quadratic assumptions in traditional finance.

\textbf{Transaction cost as temporal regularization.} Finally, we address the cost of trading. Classical frameworks often use static $L_2$ regularization to prevent overfitting or $L_1$ regularization to enforce sparsity \cite{brodie2009sparse}. However, these spatial constraints fail to penalize policy volatility: a network can satisfy strict sparsity constraints while still oscillating wildly between assets over time.

To enforce temporal consistency, we explicitly model transaction costs via the term $\frac{\gamma}{2}\|w_t - w_{t-1}^+\|_1$. Unlike static penalties, this targets the velocity of the portfolio turnover. Crucially, costs are computed relative to the \textit{drifted} portfolio $w_{t-1}^+$, ensuring that passive price movements do not incur penalties. Let $w_{t-1}$ be the weights at $t-1$. Over the interval $(t-1, t]$, simple asset returns $y_t$ cause these weights to evolve naturally into the drifted weights $w_{t-1}^{+} = (\operatorname{diag}(\mathds{1}+y_t)\,w_{t-1})/((\mathds{1}+y_t)^\top w_{t-1})$. The numerator captures asset-specific growth, while the denominator renormalizes the sum.

\begin{remark}[Cost normalization] The factor $1/2$ normalizes the $L_1$ distance to represent \textit{one-way turnover} (the fraction of portfolio value actively replaced). Thus, $\gamma$ represents the effective cost per unit of turnover.
\end{remark}

This formulation has a profound effect on the learning dynamics. By including this cost directly in the objective, the gradient signal penalizes high-frequency control noise. The policy learns a valid economic trade-off: it naturally creates a \emph{hysteresis effect}, where the portfolio weights exhibit inertia against small market fluctuations. This recovers the classical \emph{no-trade zone} behavior predicted by \citet{davis1990portfolio}: rebalancing is only performed if the expected Sharpe ratio gain strictly outweighs the explicit cost of the action.

\subsubsection{Descriptive Statistics}
\label{subsubsec:descriptive-statistics}

In Table~\ref{tab:summary_stats}, we provide descriptive statistics for the asset universes included in our analysis. We evaluate performance across three distinct portfolios, each representing a different difficulty level and correlation structure:
\begin{itemize}[leftmargin=*, nosep]
    \item \textit{Global:} The Global portfolio consists of 11 highly liquid exchange-traded products that provide broad exposure to global equities, fixed income, commodities, and real estate. It represents a macro-allocation problem characterized by lower average cross-asset correlation.
    \item \textit{Sectors:} The Sectors dataset comprises the 11 primary Global Industry Classification Standard (GICS) sectors. It serves as a standard industry benchmark with high inter-asset correlation, effectively representing the broader US equity market.
    \item \textit{Liquid:} The Liquid collection contains a selection of 50 high-volume individual stocks. This universe captures the higher volatility and idiosyncratic risk associated with individual stock picking.
\end{itemize}

Returns are derived from Yahoo Finance adjusted closing prices to account for corporate actions, such as splits and dividends, and are partitioned chronologically to strictly prevent look-ahead bias.\footnote{We note that survivorship bias is not compensated for within our analysis, and assets with missing data are omitted from the study.} The training set spans from 2010/01/01 to 2018/12/31, covering the post-financial-crisis recovery. The validation set, used for hyperparameter tuning, covers 2019/01/01 to 2020/12/31. The test set extends from 2021/01/01 to 2025/12/31.

Table \ref{tab:summary_stats} reports the annualized mean and median returns alongside annualized volatility. To fully characterize the risk profile and market dynamics, we also include the Maximum Drawdown (Max DD), the average pairwise correlation between assets ($\bar{\varrho}$), the lag-1 autocorrelation (AC(1)), and the Kolmogorov-Smirnov $p$-value ($p_{\text{KS}}$). As shown in the table, the test period exhibits significantly higher volatility and maximum drawdown compared to the training period, reflecting more turbulent market conditions. This distribution shift is quantified by $p_\text{KS}$: for the Global, Liquid, and Sectors universes, $p_\text{KS} < 0.05$ in the test set, rejecting the null hypothesis that the test data follows the training distribution. This confirms that our evaluation effectively tests the model's generalization to out-of-distribution market regimes.

\begin{table}[h]
\centering
\small
\caption{Summary statistics. Returns, Volatility, and Drawdown are in percentages. $\bar{\varrho}$ is average correlation. $p_\text{KS}$ tests for regime shift.}
\label{tab:summary_stats}
\begin{tabular}{lcccccccc}
\toprule
Universe & Split & $\mu\,(\%)$ & Median $(\%)$ & $\sigma\,(\%)$ & MaxDD $(\%)$ & $\bar{\varrho}$ & AC(1) & $p_\text{KS}$ \\
\midrule
Global & Train & $7.30$ & $11.84$ & $17.47$ & $-19.37$ & $0.35$ & $-0.037$ &  \\
Global & Val & $7.46$ & $14.22$ & $15.68$ & $-13.87$ & $0.32$ & $-0.020$ & $0.607$ \\
Global & Test & $9.30$ & $13.60$ & $22.34$ & $-28.46$ & $0.46$ & $-0.045$ & $0.011$ \\
\addlinespace
Liquid & Train & $19.12$ & $14.72$ & $24.98$ & $-16.11$ & $0.38$ & $-0.021$ &  \\
Liquid & Val & $19.71$ & $30.76$ & $25.40$ & $-18.11$ & $0.39$ & $-0.029$ & $0.319$ \\
Liquid & Test & $16.95$ & $16.93$ & $31.44$ & $-30.56$ & $0.40$ & $-0.088$ & $0.019$ \\
\addlinespace
Sectors & Train & $13.46$ & $20.17$ & $17.05$ & $-18.61$ & $0.70$ & $-0.037$ &  \\
Sectors & Val & $9.75$ & $26.72$ & $17.15$ & $-18.47$ & $0.58$ & $-0.003$ & $0.826$ \\
Sectors & Test & $13.99$ & $22.01$ & $24.57$ & $-36.91$ & $0.65$ & $-0.107$ & $0.007$ \\
\bottomrule
\end{tabular}
\end{table}

\subsubsection{Implementation Details}
In the portfolio optimization setting, we augment raw asset returns with rolling volatility and rolling correlation to a cross-sectional market proxy (defined as the mean return across the asset universe). These features are computed over the lookback horizon $H$ and normalized using training-set statistics to prevent look-ahead bias. We acknowledge that achieving state-of-the-art financial performance typically requires a broader set of predictive signals (alphas). However, our primary objective is to evaluate the stability of the learning process under the high noise-to-signal ratio characteristic of financial time series, rather than to maximize the Sharpe ratio through signal engineering.

\paragraph{Anchor selection.}
For portfolio constraints on the standard simplex, we use the barycenter
\(u_0=\mathds{1}/N\). For capped-simplex constraints with cap vector \(c\), we use
\[
    (u_0)_i = \frac{c_i}{\sum_j c_j}\quad \forall i.
\]
Provided that the capped simplex has nonempty relative interior, this point lies in the relative interior of the feasible set. The construction is deterministic, depends only on the constraint set, and reduces to the barycenter whenever the cap vector is constant.

\paragraph{Smooth transaction-cost approximation.}
As discussed in Sec.~\ref{subsec:portfolio-optimization}, the $L_1$ transaction-cost term in~\eqref{eq:sharpe_opt} is non-differentiable at zero, which can make gradient-based optimization less stable. We therefore replace the exact $L_1$ norm with a smooth Charbonnier/Pseudo-Huber approximation,
\begin{equation}
L_\delta(w_t, w_{t-1}^+) = \sum_{i=1}^N \left( \sqrt{\delta^2 + (w_{i,t} - w_{i,t-1}^+)^2} - \delta \right),
\end{equation}
where $\delta>0$ controls the smoothness of the approximation. As $\delta \to 0$, this penalty converges pointwise to \(\|w_t-w_{t-1}^+\|_1\), thereby preserving the intended transaction-cost structure while providing stable gradients during training.

\subsubsection{Additional Empirical Results}
In Section~\ref{sec:numerical-experiments}, we presented results for the Sharpe ratio maximization formulation defined in \eqref{eq:sharpe_opt}. Table~\ref{tab:portfolio-results-lstm} reports LSTM results across all three asset universes, while Table~\ref{tab:portfolio-results-transformer} reports the corresponding Transformer results. Consistent with our findings in Section~\ref{sec:numerical-experiments}, our \emph{soft-radial projection} achieves strong financial performance across different datasets and network architectures. We further compare against GLinSAT~\citep{zeng2024glinsat} and T-SKM-Net~\citep{zhu2026t} on the Global dataset in Table~\ref{tab:additional-portfolio-baselines}.

\begin{table}[h]
\centering
\small
\caption{Additional portfolio optimization results for GLinSAT and T-SKM-Net on the \textit{Global} dataset using an LSTM policy, aggregated over five seeds (mean $\pm$ std). Both methods enforce the capped-simplex constraints.}
\label{tab:additional-portfolio-baselines}
\begin{tabular}{lclcccc}
\toprule
Universe & Set & Method & SR (Net)$\uparrow$ & Turnover$\downarrow$ & Training [s]$\downarrow$ & Inference [ms]$\downarrow$\\
\midrule
\multirow{2}{*}{Global} & \multirow{2}{*}{$\mathcal{C}$} & GLinSAT & $0.68\,(0.03)$ & $0.02\,(0.01)$ & $1815.99\,(435.36)$ & $361.51\,(90.96)$ \\
& & T-SKM-Net & $0.43\,(0.08)$ & $0.18\,(0.01)$ & $228.72\,(0.73)$ & $35.90\,(1.25)$ \\
\bottomrule
\end{tabular}
\end{table}

\begin{table}[h]
\centering
\small
\caption{Portfolio optimization results for the LSTM policy over five seeds (mean $\pm$ std). The Global results reproduce Table~\ref{tab:portfolio_results} for completeness.}
\begin{tabular}{lclcccc}
\toprule
Universe & Set & Method & SR (Net)$\uparrow$ & Turnover$\downarrow$ & Training [s]$\downarrow$ & Inference [ms]$\downarrow$\\
\midrule
Global & $\Delta$ & Softmax & $0.65\,(0.09)$ & $0.11\,(0.02)$ & $11.16\,(0.19)$ & $0.34\,(0.08)$ \\
\midrule
\multirow{5}{*}{Global} & \multirow{5}{*}{$\mathcal{C}$} & O-Proj & $0.44\,(0.13)$ & $0.16\,(0.02)$ & $46.29\,(0.97)$ & $5.81\,(0.19)$ \\
& & DC3 & $0.25\,(0.18)$ & $0.23\,(0.01)$ & $23.52\,(0.41)$ & $8.44\,(0.36)$ \\
& & HardNet & $0.46\,(0.21)$ & $0.16\,(0.04)$ & $17.27\,(0.92)$ & $6.87\,(0.10)$ \\
& & RAYEN & $0.60\,(0.06)$ & $0.11\,(0.01)$ & $14.09\,(0.25)$ & $0.68\,(0.01)$ \\
& & SRP & $0.70\,(0.06)$ & $0.04\,(0.02)$ & $14.73\,(0.14)$ & $0.76\,(0.01)$ \\
\midrule
Liquid & $\Delta$ & Softmax & $0.70\,(0.13)$ & $0.14\,(0.02)$ & $10.94\,(0.22)$ & $0.33\,(0.07)$ \\
\midrule
\multirow{5}{*}{Liquid} & \multirow{5}{*}{$\mathcal{C}$} & O-Proj & $0.64\,(0.07)$ & $0.22\,(0.01)$ & $56.01\,(1.18)$ & $6.61\,(1.17)$ \\
& & DC3 & $0.59\,(0.07)$ & $0.23\,(0.01)$ & $27.84\,(0.14)$ & $6.62\,(0.82)$ \\
& & HardNet & $0.59\,(0.10)$ & $0.20\,(0.00)$ & $29.49\,(0.35)$ & $7.52\,(1.18)$ \\
& & RAYEN & $0.85\,(0.05)$ & $0.12\,(0.00)$ & $14.97\,(1.50)$ & $0.71\,(0.10)$ \\
& & SRP & $0.87\,(0.05)$ & $0.08\,(0.01)$ & $14.56\,(0.28)$ & $0.75\,(0.10)$ \\
\midrule
Sectors & $\Delta$ & Softmax & $0.70\,(0.12)$ & $0.12\,(0.04)$ & $11.60\,(0.30)$ & $0.33\,(0.07)$ \\
\midrule
\multirow{5}{*}{Sectors} & \multirow{5}{*}{$\mathcal{C}$} & O-Proj & $0.82\,(0.01)$ & $0.01\,(0.00)$ & $51.53\,(1.34)$ & $6.11\,(0.87)$ \\
& & DC3 & $0.81\,(0.02)$ & $0.01\,(0.00)$ & $28.71\,(0.73)$ & $6.92\,(1.13)$ \\
& & HardNet & $0.82\,(0.01)$ & $0.01\,(0.00)$ & $26.44\,(0.39)$ & $6.13\,(0.69)$ \\
& & RAYEN & $0.82\,(0.01)$ & $0.01\,(0.00)$ & $14.31\,(0.43)$ & $0.66\,(0.04)$ \\
& & SRP & $0.83\,(0.00)$ & $0.00\,(0.00)$ & $14.21\,(0.18)$ & $0.74\,(0.15)$ \\
\bottomrule
\end{tabular}
\label{tab:portfolio-results-lstm}
\end{table}

\begin{table}[h]
\centering
\small
\caption{Portfolio optimization results for the Transformer policy over five seeds (mean $\pm$ std).}
\begin{tabular}{lclcccc}
\toprule
Universe & Set & Method & SR (Net)$\uparrow$ & Turnover$\downarrow$ & Training [s]$\downarrow$ & Inference [ms]$\downarrow$\\
\midrule
Global & $\Delta$ & Softmax & $0.32\,(0.18)$ & $0.21\,(0.09)$ & $16.20\,(0.10)$ & $0.73\,(0.14)$ \\
\midrule
\multirow{5}{*}{Global} & \multirow{5}{*}{$\mathcal{C}$} & O-Proj & $0.51\,(0.20)$ & $0.10\,(0.07)$ & $51.52\,(1.35)$ & $6.08\,(0.52)$ \\
& & DC3 & $0.52\,(0.10)$ & $0.17\,(0.04)$ & $40.54\,(0.40)$ & $8.57\,(0.52)$ \\
& & HardNet & $0.48\,(0.07)$ & $0.04\,(0.07)$ & $34.63\,(1.24)$ & $7.25\,(0.91)$ \\
& & RAYEN & $0.46\,(0.11)$ & $0.10\,(0.02)$ & $18.34\,(0.40)$ & $1.07\,(0.01)$ \\
& & SRP & $0.56\,(0.09)$ & $0.15\,(0.07)$ & $18.99\,(0.44)$ & $1.09\,(0.03)$ \\
\midrule
Liquid & $\Delta$ & Softmax & $-0.03\,(0.25)$ & $0.41\,(0.14)$ & $17.24\,(0.21)$ & $0.73\,(0.04)$ \\
\midrule
\multirow{5}{*}{Liquid} & \multirow{5}{*}{$\mathcal{C}$} & O-Proj & $0.68\,(0.14)$ & $0.11\,(0.03)$ & $56.76\,(0.48)$ & $6.72\,(0.97)$ \\
& & DC3 & $0.64\,(0.15)$ & $0.16\,(0.05)$ & $34.17\,(1.32)$ & $7.55\,(1.04)$ \\
& & HardNet & $0.73\,(0.18)$ & $0.06\,(0.03)$ & $35.49\,(0.51)$ & $7.55\,(1.21)$ \\
& & RAYEN & $0.79\,(0.07)$ & $0.10\,(0.03)$ & $20.92\,(0.51)$ & $1.11\,(0.08)$ \\
& & SRP & $0.80\,(0.05)$ & $0.10\,(0.02)$ & $20.18\,(0.21)$ & $1.14\,(0.12)$ \\
\midrule
Sectors & $\Delta$ & Softmax & $0.67\,(0.15)$ & $0.01\,(0.01)$ & $17.86\,(0.54)$ & $0.80\,(0.15)$ \\
\midrule
\multirow{5}{*}{Sectors} & \multirow{5}{*}{$\mathcal{C}$} & O-Proj & $0.80\,(0.02)$ & $0.01\,(0.01)$ & $57.20\,(1.03)$ & $6.60\,(1.01)$ \\
& & DC3 & $0.80\,(0.01)$ & $0.01\,(0.00)$ & $46.94\,(0.85)$ & $9.83\,(1.30)$ \\
& & HardNet & $0.82\,(0.01)$ & $0.01\,(0.00)$ & $35.21\,(0.32)$ & $7.95\,(1.15)$ \\
& & RAYEN & $0.80\,(0.01)$ & $0.01\,(0.00)$ & $21.11\,(0.57)$ & $1.12\,(0.13)$ \\
& & SRP & $0.83\,(0.00)$ & $0.00\,(0.00)$ & $21.05\,(0.38)$ & $1.35\,(0.24)$ \\
\bottomrule
\end{tabular}
\label{tab:portfolio-results-transformer}
\end{table}

\subsection{Extended Analysis: Ride-Sharing Dispatch}
\label{subsec:extended-analysis-ride-sharing-dispatch}

Our taxi dataset is derived from the publicly available records of the NYC Taxi and Limousine Commission (TLC) for the period of January to June 2024. To ensure a robust evaluation of temporal generalization, the data is split chronologically into 70\% training, 15\% validation, and 15\% testing sets. To capture the multi-scale periodicity of urban mobility, we represent temporal states using cyclical sine-cosine encodings for both daily and weekly cycles. Unlike categorical indicators that create artificial boundaries at midnight, this continuous representation allows the policy network $g_\theta$ to learn smooth, time-dependent transition boundaries. This enables the model to proactively align fleet distribution with anticipated demand peaks, maximizing the global served rate within the constraints of the time-varying supply $S_t$.

\paragraph{Anchor selection.}
For the normalized taxi capped simplex \(\{u\ge0:\mathds{1}^\top u=1,\;u_i\le\kappa\;\forall i\}\), the reported settings use \(N=150\) zones and satisfy \(N\kappa>1\). Hence the barycenter \(\bar u_0=\mathds{1}/N\) is strictly feasible. The anchor in allocation coordinates is \(u_{0,t}=S_t\bar u_0\).

In Section~\ref{sec:numerical-experiments} (Table~\ref{tab:taxi-results}), we presented results for an MLP model subject to the scaled capped simplex with $\kappa=0.1$. Here, we introduce a more challenging experimental setting with $\kappa=0.02$. This tighter constraint creates a significantly more difficult allocation problem, requiring the distribution of the fleet across all $N=150$ zones with strictly limited per-zone capacity. For completeness, we also provide results for projection onto the standard simplex. While this baseline serves as a reference point consistent with Table~\ref{tab:taxi-results}, we observe a clear deterioration in performance among the constrained models under this stricter regime, characterized by a lower volume of served customers (\textit{cf.} Table~\ref{tab:taxi-results-kappa-lower}). Consistent with the $\kappa=0.1$ regime, soft-radial projection, RAYEN, DC3, and HardNet all achieve the best observed served rate ($0.69$), while orthogonal projection attains a lower served rate. These results extend the favorable performance of soft-radial projection to more restrictive capacity constraints.

\begin{table}[h]
\caption{Ride-sharing dispatch results aggregated over five seeds (mean $\pm$ std).}
\centering
\small
\begin{tabular}{clccc}
\toprule
Set & Method & Served Rate$\uparrow$ & Training [s]$\downarrow$ & Inference [ms]$\downarrow$\\
\midrule
$\Delta$ & Softmax & $0.84\,(0.00)$ & $13.05\,(1.47)$ & $0.36\,(0.04)$\\
\midrule
\multirow{5}{*}{$\kappa=0.02$} & O-Proj & $0.64\,(0.00)$ & $74.92\,(2.93)$ & $9.65\,(0.81)$\\
& DC3 & $0.69\,(0.00)$ & $42.49\,(1.35)$ & $12.53\,(1.55)$\\
& HardNet & $0.69\,(0.00)$ & $42.65\,(0.55)$ & $10.81\,(1.24)$\\
& RAYEN & $0.69\,(0.00)$ & $16.33\,(0.60)$ & $0.85\,(0.06)$\\
& SRP & $0.69\,(0.00)$ & $15.24\,(0.37)$ & $0.92\,(0.09)$\\
\bottomrule
\end{tabular}
\label{tab:taxi-results-kappa-lower}
\end{table}

\subsection{Network-Constrained Quadratic Programs}
\label{subsec:structured-affine-benchmarks}

To evaluate the layer's capacity to handle structured polyhedral sets, we benchmark on two network-constrained operational tasks: Min-Cost Flow and DC Optimal Power Flow (DC-OPF). Unlike standard simplex projections, feasibility in these domains enforces strict physical conservation laws and asset capacities, yielding complex intersections of affine equalities and inequalities. 

Both tasks are formulated as amortized convex quadratic programs. A neural policy observes a temporal window of dynamic cost features and predicts a candidate control vector in reduced coordinates, which the constraint layer subsequently projects into the feasible polytope. This setup strictly isolates the layer's ability to preserve informative gradients across active constraint boundaries while benchmarking the resulting optimality gap, strict feasibility, and forward-pass latency.

\subsubsection{Min-Cost Flow}

The min-cost flow benchmark represents a capacitated allocation problem on a directed graph $\mathcal G=(\mathcal V,\mathcal E)$. The decision variable is the vector of edge flows $x\in\mathbb R^m$, where $m=|\mathcal E|$. Feasibility requires flow conservation at every node: apart from prescribed sources and sinks, the policy cannot create or destroy mass. It must also respect lower and upper capacity bounds on every edge. Across samples, the graph, supplies, demands, and capacities remain fixed, while the quadratic and linear edge costs vary.

For each sample $t$, we consider the convex quadratic min-cost flow problem
\begin{equation}
    \label{eq:mcf_final}
    \begin{array}{cl}
        \displaystyle \min_{x \in \mathbb{R}^{m}} & \displaystyle \frac{1}{2} x^\top Q_t x + c_t^\top x \\[1em]
        \text{s.t.} & \displaystyle \sum_{e \in \delta^+(i)} x_e - \sum_{e \in \delta^-(i)} x_e = b_i, \quad \forall i \in \mathcal{V}, \\[1em]
        & \displaystyle \underline{x}_e \le x_e \le \overline{x}_e, \quad \forall e \in \mathcal{E}.
    \end{array}
\end{equation}
Here, $\delta^+(i)$ and $\delta^-(i)$ denote the sets of outgoing and incoming edges at node $i$. The vector $b \in \mathbb R^{|\mathcal V|}$ represents node supplies and demands and satisfies $\mathds{1}^\top b = 0$. The bounds $\underline x, \overline x \in \mathbb R^{m}$ represent edge capacities. These conditions strictly define the feasible set $\mathcal C$.

\textbf{Sample-dependent objective generation.}
Let $q^{(0)}, c^{(0)} \in \mathbb R^{m}$ denote base edge costs with $q^{(0)}_e > 0$. To avoid trivial, unstructured instances, we generate a temporally correlated cost sequence using an autoregressive process defined by a global latent factor $g_t = \rho g_{t-1} + \sigma_g \varepsilon_t$ and edge-specific factors $h_{t,e} = \rho h_{t-1,e} + \sigma_h \eta_{t,e}$ for $e\in\mathcal E$, where $\varepsilon_t$ and $\eta_{t,e}$ are drawn from standard normal distributions. Here, \(\rho\in[0,1)\) controls temporal persistence, while \(\sigma_g\) and \(\sigma_h\) set the volatility of the global and component-specific shocks. The sample-dependent cost coefficients are generated via multiplicative log-normal perturbations, $q_{t,e} = q^{(0)}_e \exp(g_t + h_{t,e})$ and $c_{t,e} = c^{(0)}_e \exp(g_t + h_{t,e})$, with the quadratic term set as $Q_t = \operatorname{diag}(q_t)$. This dynamic construction prevents the network from memorizing a static optimal solution, formulating the task as a true amortized prediction problem.

\textbf{Observed features and prediction task.}
Because costs span multiple orders of magnitude and the min-cost flow construction uses positive edge costs, the policy network observes log-transformed quadratic and linear cost coefficients. Given a small numerical stability constant $\epsilon > 0$, we define the stacked feature vector $\phi_t = \bigl(\log(q_t + \epsilon\mathds{1}),\; \log(c_t + \epsilon\mathds{1})\bigr) \in \mathbb{R}^{2m}$. A neural network maps a lookback window of these features to a candidate flow in reduced coordinates, which the projection layer subsequently maps into $\mathcal{C}$.

\begin{table}[h]
\small
\centering
\caption{Min-cost-flow results on synthetic transport graphs with
\(|\mathcal V|\in\{8,12\}\) nodes, aggregated over five seeds
(mean $\pm$ std).}
\begin{tabular}{clccc}
\toprule
$|\mathcal V|$ & Method & Rel. gap$\downarrow$ & Training [s]$\downarrow$ & Inference [ms]$\downarrow$\\
\midrule
8 & O-Proj & $0.10\,(0.03)$ & $789.49\,(44.37)$ & $456.27\,(125.75)$\\
8 & RAYEN & $1.51\,(0.09)$ & $1.72\,(0.13)$ & $1.09\,(0.14)$ \\
8 & SRP & $1.52\,(0.11)$ & $1.52\,(0.28)$ & $0.96\,(0.19)$ \\
\midrule
12 & O-Proj & $0.06\,(0.04)$ & $1354.76\,(65.17)$ & $695.99\,(171.10)$ \\
12 & RAYEN & $1.11\,(0.09)$ & $1.85\,(0.09)$ & $1.08\,(0.14)$ \\
12 & SRP & $1.08\,(0.09)$ & $2.03\,(0.14)$ & $1.09\,(0.12)$\\
\bottomrule
\end{tabular}
\label{tab:flow_results}
\end{table}

\subsubsection{DC Optimal Power Flow}

The DC-OPF benchmark evaluates amortized dispatch using the linearized DC power-flow approximation on standard PGLib test cases \cite{babaeinejadsarookolaee2019power}. The decision variable $x\in\mathbb R^n$ represents the vector of generator outputs. A feasible dispatch must exactly match total demand, respect generation operating ranges, and maintain induced line flows below their thermal limits. These residuals have direct operational meaning: mismatches violate power balance, while line-limit violations correspond to overloaded transmission lines.

\textbf{Fixed-demand variant.}
The first variant keeps the PGLib network, base load vector, and all operational limits fixed across samples. The quadratic and linear generation-cost coefficients vary over time (using the same autoregressive construction as in min-cost flow, with scenario-specific parameter values). For each sample $t$, we solve
\begin{equation}
    \label{eq:dcopf_final}
    \begin{array}{cl}
        \displaystyle \min_{x \in \mathbb{R}^{n}} & \displaystyle \frac{1}{2} x^\top Q_t x + c_t^\top x \\[1em]
        \text{s.t.} & \displaystyle \mathds{1}^\top x = D, \\[1em]
        & \displaystyle \underline{x} \le x \le \overline{x}, \\[1em]
        & \displaystyle -\overline{f} \le F(Cx - d) \le \overline{f}.
    \end{array}
\end{equation}
Here, $D \in \mathbb R$ denotes total demand, $C \in \mathbb R^{b \times n}$ maps generation to nodal injections, $d \in \mathbb R^{b}$ is the demand vector, $F \in \mathbb R^{\ell \times b}$ maps nodal injections to line flows, and $\overline f \in \mathbb R^{\ell}$ represents line capacities. 

\begin{table}[h]
\small
\centering
\caption{DC-OPF results aggregated over five seeds (mean $\pm$ std).}
\begin{tabular}{llccc}
\toprule
Scenario & Method & Rel. gap$\downarrow$ & Training [s]$\downarrow$ & Inference [ms]$\downarrow$\\
\midrule
14 Bus & O-Proj & $0.00\,(0.00)$ & $12.88\,(1.80)$ & $13.60\,(3.40)$\\
14 Bus & RAYEN & $0.00\,(0.00)$ & $4.62\,(0.84)$ & $3.19\,(1.45)$\\
14 Bus & SRP & $0.00\,(0.00)$ & $4.90\,(0.96)$ & $3.07\,(1.29)$\\
\midrule
118 Bus & O-Proj & $0.01\,(0.00)$ & $7473.10\,(106.65)$ & $2821.26\,(64.62)$\\
118 Bus & RAYEN & $0.03\,(0.00)$ & $1.94\,(0.08)$ & $0.68\,(0.06)$\\
118 Bus & SRP & $0.03\,(0.00)$ & $3.67\,(1.18)$ & $0.85\,(0.09)$\\
\bottomrule
\end{tabular}
\label{tab:dcopf-results}
\end{table}

\paragraph{Dynamic-demand variant.}
To test stability under a dynamic feasible polytope, we allow the total load level to vary over time. The base demand vector is scaled as $d_t=s_t d$, where the load scale $s_t$ follows a clipped log-autoregressive process $\log s_t = \rho \log s_{t-1}+\sigma \xi_t$, bounded within the interval $[\underline{s},\bar{s}]$. Here, \(\rho\) controls persistence in aggregate demand and \(\sigma\) controls the magnitude of load fluctuations. If an extreme sampled scale does not admit a strict interior point, we linearly revert it toward the nominal scale $s=1$ until strict feasibility is recovered. Consequently, the feasible set $\mathcal C_t$ becomes sample-dependent in both the power-balance and line-flow constraints.

The policy observes log-transformed quadratic costs and signed-log-transformed linear costs, augmented with $\log s_t$ in the dynamic-demand setting. We compute a strictly feasible particular point $x_p(s_t)$ to parameterize the equality set as $x=x_p(s_t)+B\xi$, where the columns of $B\in\mathbb R^{n\times(n-1)}$ form a basis for the null space of the power-balance row $\mathds{1}^\top$ and $\xi\in\mathbb R^{n-1}$ denotes reduced coordinates. For radial methods, this point also serves as the ray anchor; for projection and repair methods, it serves only as the affine offset. To avoid solving an anchor problem at every training step, these strict particular points are precomputed over a uniform grid (4 points) of load scales and linearly interpolated during training.

Table~\ref{tab:dcopf-dynamic-results} reports the dynamic-demand IEEE 14 Bus case for RAYEN and SRP using a 4-point anchor grid. We omit orthogonal projection in this dynamic variant because each changing feasible polytope would require solving a separate projection problem at inference time, making the comparison dominated by solver overhead rather than by the geometry of the learned feasibility layer. Both methods remain feasible across the test set and achieve comparable objective gaps, while retaining substantially lower latency than the orthogonal-projection baseline in the fixed-demand setting.

\begin{table}[h]
\caption{Dynamic-demand DC-OPF results aggregated over five seeds (mean $\pm$ std).}
\footnotesize
\centering
\begin{tabular}{llccccc}
\toprule
Scenario & Method & Rel. gap$\downarrow$ & Training [s]$\downarrow$ & Inference [ms]$\downarrow$ & Feas.~Rate & Anchor Precompute [ms]$\downarrow$ \\
\midrule
14 Bus & RAYEN & $0.01\,(0.00)$ & $2.50\,(1.78)$ & $1.88\,(1.61)$ & $1.00\,(0.00)$ & $5.79\,(0.51)$ \\
14 Bus & SRP & $0.00\,(0.00)$ & $3.38\,(2.63)$ & $2.48\,(1.80)$ & $1.00\,(0.00)$ & $7.87\,(4.84)$\\
\bottomrule
\end{tabular}
\label{tab:dcopf-dynamic-results}
\end{table}

\newpage
\section*{NeurIPS Paper Checklist}

\begin{enumerate}

\item {\bf Claims}
    \item[] Question: Do the main claims made in the abstract and introduction accurately reflect the paper's contributions and scope?
    \item[] Answer: \answerYes{} 
    \item[] Justification: See Section \ref{sec:intro}.

\item {\bf Limitations}
    \item[] Question: Does the paper discuss the limitations of the work performed by the authors?
    \item[] Answer: \answerYes{} 
    \item[] Justification: See Section \ref{sec:conclusion}.

\item {\bf Theory assumptions and proofs}
    \item[] Question: For each theoretical result, does the paper provide the full set of assumptions and a complete (and correct) proof?
    \item[] Answer: \answerYes{} 
    \item[] Justification:  Please refer to Sections \ref{sec:soft-projection}--\ref{sec:const-ua}, as well as the Appendix \ref{app:proofs}.

    \item {\bf Experimental result reproducibility}
    \item[] Question: Does the paper fully disclose all the information needed to reproduce the main experimental results of the paper to the extent that it affects the main claims and/or conclusions of the paper (regardless of whether the code and data are provided or not)?
    \item[] Answer: \answerYes{} 
    \item[] Justification: See Details in Section \ref{sec:numerical-experiments} and Appendix \ref{app:extended-numerical-results}.

\item {\bf Open access to data and code}
    \item[] Question: Does the paper provide open access to the data and code, with sufficient instructions to faithfully reproduce the main experimental results, as described in supplemental material?
    \item[] Answer: \answerYes{} 
    \item[] Justification: The source code is accessible at \url{https://anonymous.4open.science/r/soft-radial-projection-e2e-neurips/}.

\item {\bf Experimental setting/details}
    \item[] Question: Does the paper specify all the training and test details (e.g., data splits, hyperparameters, how they were chosen, type of optimizer) necessary to understand the results?
    \item[] Answer: \answerYes{} 
    \item[] Justification: Refer to Section \ref{sec:numerical-experiments} and Appendix \ref{app:extended-numerical-results}.

\item {\bf Experiment statistical significance}
    \item[] Question: Does the paper report error bars suitably and correctly defined or other appropriate information about the statistical significance of the experiments?
    \item[] Answer: \answerYes{} 
    \item[] Justification: See Section \ref{sec:numerical-experiments} and Appendix \ref{app:extended-numerical-results}.

\item {\bf Experiments compute resources}
    \item[] Question: For each experiment, does the paper provide sufficient information on the computer resources (type of compute workers, memory, time of execution) needed to reproduce the experiments?
    \item[] Answer: \answerYes{} 
    \item[] Justification: See Section \ref{sec:numerical-experiments} and Appendix \ref{app:extended-numerical-results}.
    
\item {\bf Code of ethics}
    \item[] Question: Does the research conducted in the paper conform, in every respect, with the NeurIPS Code of Ethics \url{https://neurips.cc/public/EthicsGuidelines}?
    \item[] Answer: \answerYes{} 
    \item[] Justification: Our work conforms in every respect with the NeurIPS Code of Ethics.

\item {\bf Broader impacts}
    \item[] Question: Does the paper discuss both potential positive societal impacts and negative societal impacts of the work performed?
    \item[] Answer: \answerNA{} 
    \item[] Justification: The work addresses the fundamental problem of modeling optimization problems subject to constraints.
    
\item {\bf Safeguards}
    \item[] Question: Does the paper describe safeguards that have been put in place for responsible release of data or models that have a high risk for misuse (e.g., pre-trained language models, image generators, or scraped datasets)?
    \item[] Answer: \answerNA{} 
    \item[] Justification: Our paper does not pose a risk of misuse.

\item {\bf Licenses for existing assets}
    \item[] Question: Are the creators or original owners of assets (e.g., code, data, models), used in the paper, properly credited and are the license and terms of use explicitly mentioned and properly respected?
    \item[] Answer: \answerYes{} 
    \item[] Justification: The data and frameworks used in this paper are cited; see Section \ref{sec:numerical-experiments} and Appendix \ref{app:extended-numerical-results}.

\item {\bf New assets}
    \item[] Question: Are new assets introduced in the paper well documented and is the documentation provided alongside the assets?
    \item[] Answer: \answerYes{} 
    \item[] Justification: We provide an anonymized repository containing the code used to validate our results.

\item {\bf Crowdsourcing and research with human subjects}
    \item[] Question: For crowdsourcing experiments and research with human subjects, does the paper include the full text of instructions given to participants and screenshots, if applicable, as well as details about compensation (if any)? 
    \item[] Answer: \answerNA{} 
    \item[] Justification: The research does not involve research with human subjects.

\item {\bf Institutional review board (IRB) approvals or equivalent for research with human subjects}
    \item[] Question: Does the paper describe potential risks incurred by study participants, whether such risks were disclosed to the subjects, and whether Institutional Review Board (IRB) approvals (or an equivalent approval/review based on the requirements of your country or institution) were obtained?
    \item[] Answer: \answerNA{} 
    \item[] Justification: The research does not involve research with human subjects.

\item {\bf Declaration of LLM usage}
    \item[] Question: Does the paper describe the usage of LLMs if it is an important, original, or non-standard component of the core methods in this research? Note that if the LLM is used only for writing, editing, or formatting purposes and does \emph{not} impact the core methodology, scientific rigor, or originality of the research, declaration is not required.
    \item[] Answer: \answerNA{} 
    \item[] Justification: The methods of this research do not involve LLMs.

\end{enumerate}

\end{document}